\documentclass[letterpaper,11 pt]{article} 
\usepackage{amsmath,amsfonts,amssymb,color}
\usepackage{mathtools}
\usepackage{dsfont}
\usepackage[dvipsnames]{xcolor}
\definecolor{mygreen}{rgb}{0.0, 0.5, 0.0}
\definecolor{winered}{rgb}{0.8,0,0}
\definecolor{myblue}{rgb}{0,0,0.8}
\usepackage{tikz}
\usepackage{amsthm}
\usepackage{empheq}

\newtheorem{definition}{Definition}
\newtheorem{theorem}{Theorem}
\newtheorem{lemma}{Lemma}

\newtheorem{remark}{Remark}
\newtheorem{assumption}{Assumption}
\DeclarePairedDelimiter\ceil{\lceil}{\rceil}
\DeclarePairedDelimiter\floor{\lfloor}{\rfloor}

\DeclareMathOperator*{\argmax}{\arg\!\max}

\DeclarePairedDelimiterX{\norm}[1]{\lVert}{\rVert}{#1}

\usepackage{algorithm}
\usepackage{algpseudocode}
\usepackage{epsfig}
\usepackage{array}
\usepackage{multirow}
\usepackage{epstopdf}
\usepackage{tikz}
\usepackage{relsize}
\usetikzlibrary{shapes,arrows}
\usepackage[hidelinks]{hyperref}
\hypersetup{
    colorlinks=true,
    linkcolor={winered},
    citecolor={mygreen}
}
\usepackage{cite}
\usepackage[margin=1in]{geometry}
\title{\LARGE Linear Stochastic Bandits over a Bit-Constrained Channel}
\author{Aritra Mitra, Hamed Hassani, and George J. Pappas
\thanks{The authors are with the Department of Electrical and Systems Engineering, University of Pennsylvania. Email: {\tt \{amitra20, hassani, pappasg\}@seas.upenn.edu}. This work was supported by NSF Award 1837253, NSF CAREER award CIF 1943064, and the Air Force Office
of Scientific Research Young Investigator Program (AFOSR-YIP) under award FA9550-20-1-0111.}}
\date{}
\begin{document}
\maketitle
\thispagestyle{empty}
\pagestyle{empty}
\begin{abstract}
One of the primary challenges in large-scale distributed learning stems from stringent communication constraints. While several recent works address this challenge for static optimization problems, sequential decision-making under uncertainty has remained much less explored in this regard. Motivated by this gap, we introduce a new linear stochastic bandit formulation over a bit-constrained channel. 
Specifically, in our setup, an agent interacting with an environment transmits encoded  estimates of an unknown model parameter to a server over a communication channel of finite capacity. The goal of the server is to take actions based on these estimates to minimize cumulative regret. To this end, we develop a novel and general algorithmic framework that hinges on two main components: (i) an adaptive encoding mechanism that exploits statistical concentration bounds, and (ii) a decision-making principle based on confidence sets that account for encoding errors. As our main result, we prove that when the unknown model is $d$-dimensional, a channel capacity of $O(d)$ bits suffices to achieve order-optimal regret. To demonstrate the generality of our approach, we then show that the same result continues to hold for non-linear observation models satisfying standard regularity conditions. Finally, we establish that for the simpler unstructured multi-armed bandit problem, $1$ bit channel-capacity is sufficient for achieving optimal regret bounds. Overall, our work takes a significant first step towards paving the way for statistical decision-making over finite-capacity channels.  
\end{abstract}
\section{Introduction}
In modern distributed computing paradigms such as federated learning (FL), a group of agents typically interact with a parameter server to train a common statistical model. A major bottleneck in such settings is the network communication cost of uploading (potentially high-dimensional) models and gradient vectors to the server. Motivated by this emerging concern, several works draw on ideas from quantization theory \cite{seide,strom, wen,qsgd,khirirat,DIANA,ADIANA,horvathq,fedpaq,Fedcomgate}, sparsification \cite{emp1,emp2,alistarhsparse,stichsparse,reddySignSGD,reddystich,beznosikov,lin_comp,mitraNIPS21}, and rate-distortion theory \cite{mitchell} to design communication-efficient algorithms that achieve a desired level of precision while exchanging as few bits as possible. This rich body of work contributes significantly to the study of static optimization problems under communication constraints. Nonetheless,  there remains a considerable gap in our understanding of similar questions when it comes to sequential decision-making under uncertainty (e.g., bandit problems and reinforcement learning). Our primary goal in this paper is to bridge the above gap. 

A common abstraction for analyzing optimization under limited communication is one where a worker agent transmits quantized gradients to a server over a finite bit-rate communication channel \cite{mayekar,gandikota,kostina}. Inspired by this model, for our problem of interest, we introduce and study a new linear stochastic bandit formulation comprising of an agent connected to a decision-making entity (server) by a noiseless communication channel of finite capacity $B$; see Fig.~\ref{fig:Model}. The agent interacts with an environment and observes noisy rewards that depend linearly on an unknown parameter vector $\theta_* \in \mathbb{R}^d$. It then encodes and transmits finite-precision estimates of  $\theta_*$ to the server. Based on these estimates, the role of the server is to play a sequence of actions that maximizes the sum of rewards accrued over a time horizon $T$ - a performance metric captured by cumulative regret.\footnote{A formal description of our setup is provided in Section \ref{sec:model}.} Notably, the agent can only transmit encoded estimates of the parameter vector, but not the rewards themselves. The reason for this is twofold. First, our formulation is motivated by the popular federated learning framework \cite{konevcny} where due to privacy concerns, agents exchange their local models with the server instead of their raw observations. In our setup, the parameter vector is the model and the rewards are the observations. Second, our goal is to build a theory that is eventually applicable to multi-agent decision-making. For such settings, it is more natural for the server to perform fusion on the agents' local models instead of directly fusing observations that may belong to different signal spaces. 


The main technical challenge in our setup arises from the fact that the channel from the agent to the server introduces additional uncertainty into the decision-making process. Unless accounted for carefully, the instantaneous encoding errors resulting from such uncertainty can accumulate over time and lead to sub-optimal regret bounds. Given this challenge, the central question we investigate is the following. 

\vspace{1mm}
\textit{Under what conditions on the channel capacity $B$ can we achieve the order-optimal regret bound $\tilde{O}(d\sqrt{T})$?}\footnote{When the channel has infinite capacity, i.e., when $B=\infty$, $\tilde{O}(d\sqrt{T})$ regret is known to be optimal; see Chapter 24 of \cite{tor} for further details on this topic.} 

\vspace{1mm}
In this work, we rigorously answer the above question via a set of algorithmic and theoretical contributions discussed below.

$\bullet$ \textbf{Algorithmic Contributions.} For the setting of interest, we develop a novel framework for statistical decision-making under communication constraints. Our approach hinges on two main components. The first is an adaptive quantization mechanism that encodes the change (\textit{innovation})  in successive estimates of $\theta_*$ at the agent. The main intuition here is that with high probability, the gap between successive model estimates shrinks over time; as a result, the innovation signals are contained in balls of progressively smaller radii. Thus, roughly speaking, to achieve the same precision, it takes fewer bits to encode the innovation signals as compared to the model estimates (that can be of a much larger magnitude). A key feature of our encoding scheme is that the dynamic quantizer ranges are designed based on statistical concentration bounds specific to the stochastic process we study. As such, our encoding scheme is novel,  and differs significantly from standard quantization approaches for optimization. 

The second integral component of our framework is the decision-making policy at the server that comprises of two phases: (i) a pure exploration phase that facilitates the estimation of $\theta_*$, and (ii) an information-constrained exploration-exploitation phase. Specifically, in the latter phase, actions are taken based on certain  ``inflated" confidence sets that are carefully constructed: the radii of such sets need to be large enough to account for the errors induced by compression. At the same time, the compression errors need to gradually shrink to ensure that taking decisions based on ``inflated" confidence sets does not lead to sub-optimal regret bounds. Thus, the design of the  encoding scheme at the agent is tightly coupled with the  decision-making policy at the server. Notably, the construction of the confidence sets is a key  algorithmic contribution of our work that sets it apart from communication-constrained static optimization where a decision-making component is absent. We refer to our overall scheme as the Information-Constrained \texttt{LinUCB} algorithm (\texttt{IC-LinUCB}). 

$\bullet $ \textbf{Theoretical Contributions.} Our first main result (Theorem \ref{thm:ICLinUCB}) reveals that with a channel capacity $B=O(d)$ bits, \texttt{\texttt{IC-LinUCB}} guarantees a regret bound of $\tilde{O}(d\sqrt{T})$. The main implication of this result is that one can achieve order-optimal regret guarantees  with a bit-rate that is independent of the horizon $T$, and that depends only on the dimension $d$ of the unknown model $\theta_*$. As far as we are aware, this is the first result of its kind for linear stochastic bandits, and complements similar results for stochastic optimization: the authors in \cite{mayekar} recently showed that with $d$-dimensional quantized gradients, a bit-rate of $\tilde{O}(d)$ bits is sufficient for achieving the optimal optimization convergence rate.  On the technical front, we note that the proof of Theorem \ref{thm:ICLinUCB} is non-trivial, and relies on some key intermediate ideas that we outline in Section \ref{sec:guarantees}. 

To demonstrate the generality of our approach, we significantly extend our analysis to the generalized linear bandit setting that accounts for non-linear observation models \cite{GLM1,GLM2}. Once again, we establish that with a bit-rate of $O(d)$ bits, one can achieve optimal regret bounds; see Theorem \ref{thm:ICGLMUCB} for a formal statement of the result. 

Finally, we ask: \textit{When the action sets have additional structure, can we exploit such structure to achieve optimal performance with fewer than $O(d)$ bits?} To answer this question, we study a special case of the linear bandit problem where the actions are the standard orthonormal basis vectors. This setting corresponds to the multi-armed bandit (MAB) problem with a finite number of arms \cite{auer}. For this setting, we prove that with a bit-rate $B=1$, one can achieve both gap-dependent (Theorem \ref{thm:ICUCB}) and gap-independent (Theorem \ref{thm:gap_ind}) regret bounds matching those of the celebrated upper-confidence bound (\texttt{UCB}) algorithm. 


Overall, we envision that the algorithmic and analytical insights from this work will pave the way for studying more complex statistical decision-making problems in distributed and multi-agent settings under channel capacity constraints. 

\textbf{Further Related Work.} Our formulation is inspired by the classical work \cite{tatikonda} that studies the problem of stabilizing a linear time-invariant dynamical system over a bit-constrained channel. There, as in our setup, the estimation module (sensor) is separated from the decision-making module (controller) by the channel. Aside from the fact that we study a fundamentally different problem, our work departs from \cite{tatikonda} in that our setup is inherently stochastic, while the authors in \cite{tatikonda} consider a fully deterministic setting. In particular, while the state estimates encoded in \cite{tatikonda} are deterministic, the model parameter estimates that we encode in our setting are high-dimensional random vectors. 

Our work is naturally related to the seminal papers on linear stochastic bandits \cite{dani,abbasi} that introduce and analyze the \texttt{LinUCB} algorithm. The results in this paper extend those in \cite{dani,abbasi} to the communication-constrained setting that we study. In the context of multi-agent bandits \cite{landgren1,landgren2,shahrampour,kolla,dMAB3,sankararaman,martinez,dubeyICML20,dubeyNIPS20,lalitha20,chawla1,chawla2,ghoshbandits,agarwal21,FedMAB1,FedMAB3}, a body of work focuses on achieving benefits of collaboration while minimizing the number of communication rounds \cite{dMAB3, dubeyNIPS20, chawla1, agarwal21}. The main goal of these papers is to achieve desirable performance while minimizing the \textit{frequency} of communication. Our focus is orthogonal - that of studying the impact of finite-precision communication channels on the performance of bandit algorithms. As a result, our problem formulation, algorithmic techniques, and theoretical results differ considerably from the above strand of literature. 

\textbf{Notation}. Given two scalars $a$ and $b$, we use $a \vee b$ and $a \wedge b$ to represent $\max\{a,b\}$ and $\min\{a,b\}$, respectively.  For any positive integer $n$, we use $[n]$ to denote the set of integers $\{1, \ldots, n\}$. We use $\mathcal{B}_d(0,1)$ and $\mathbb{S}^{d-1}$ to represent the $d$-dimensional Euclidean ball and the $d$-dimensional Euclidean sphere, respectively, of unit radius centered at the origin. Given a matrix $A$, we use $\lambda_{\max}(A)$ and $\lambda_{\min}(A)$ to represent the largest and smallest eigenvalues, respectively, of $A$. Moreover, we use $A'$ to denote the transpose of $A$. Given two symmetric positive semi-definite matrices $A$ and $B$, we use $B \preccurlyeq A$ to imply that $A-B$ is positive semi-definite. 

\newpage
\section{Model and Problem Formulation}
\label{sec:model}

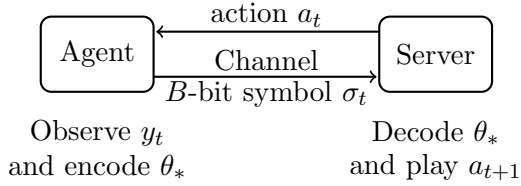
\begin{figure}[t]
\begin{center}
\begin{tikzpicture}
[->,shorten >=1pt,scale=.75,inner sep=1pt, minimum size=12pt, auto=center, node distance=3cm,
  thick, node/.style={circle, draw=black, thick},]
\tikzstyle{block1} = [rectangle, draw, fill=red!10, 
    text width=8em, text centered, rounded corners, minimum height=0.8cm, minimum width=1cm];
\node (rect) at (0,0) (A) [draw, rounded corners, minimum width= 1.5cm, minimum height=1cm] {Agent};
\node (rect) at (6,0) (S) [draw, rounded corners, minimum width= 1.5cm, minimum height=1cm] {Server}; 
\coordinate (a1) at (1, 0.4);
\coordinate (s1) at (5,0.4);
\draw[->,thick]
 (s1) -- node[pos=.5,above]{action $a_t$} (a1);  
 
\coordinate (a2) at (1,-0.4);
\coordinate (s2) at (5,-0.4);
\draw[->,thick]
 (a2) -- node[pos=.5,below] {$B$-bit symbol $\sigma_t$}  node[pos=.5,above] {Channel} (s2);

\node at (0,-1.4) {Observe $y_t$};
\node at (0,-2) {and encode $\theta_*$};
\node at (6, -1.4) {Decode $\theta_*$};
\node at (6, -2) {and play $a_{t+1}$}; 
\end{tikzpicture}
\end{center}
\caption{At each round $t$, the action $a_t$ played by the server is sent to the agent without any loss of information. The agent then observes a reward $y_t$ as per Eq. \eqref{eqn:obs_model}, encodes an estimate of the model $\theta_*$, and transmits the encoded symbol $\sigma_t$ back to the server under the $B$-bits per round channel constraint. The server performs decoding and plays the next action $a_{t+1}$.}
\label{fig:Model}
\end{figure}

We study a setting comprising of an agent and a decision-maker (server) separated by a noiseless communication channel of finite capacity; see Fig.~\ref{fig:Model}. Based on all the information acquired by the server up to time-step $t-1$, it chooses an action $a_t \in \mathcal{A}_t$ at time $t$, where $\mathcal{A}_t \subset \mathbb{R}^d$ is the feasible decision set at time $t$. The agent then receives a reward (observation) according to the following model:
\begin{equation}
    y_t= \langle \theta_*, a_t \rangle + \eta_t,
\label{eqn:obs_model}
\end{equation}
where $\{\eta_t\}$ is a sequence of i.i.d. $1$-subgaussian noise random variables.  Here, $\theta_*$ is an unknown parameter that belongs to a known compact set $\Theta \subset \mathbb{R}^d$; for each $\theta \in \Theta$, it holds that ${\Vert \theta \Vert}_2 \leq M$, where $M \geq 1$. Our performance measure of interest is the following regret metric $R_T$:
\begin{equation}
    R_T = \mathbb{E} \left[\sum_{t=1}^{T} \max_{a\in\mathcal{A}_t} \langle \theta_*, a-a_t \rangle \right],
\label{eqn:regret}
    \end{equation}
where $T$ is the time horizon. The goal of the server is to play a sequence of actions such that ${R}_T$ grows sub-linearly in $T$. When there is no loss of information from the agent to the server (i.e., in the absence of the  channel), it is well known that one can achieve  $\tilde{O}\left(d\sqrt{T}\right)$ regret by following the popular \texttt{LinUCB} algorithm \cite{abbasi}. Our \textbf{goal} in this work is to develop an algorithm that achieves the same performance subject to communication constraints that we describe next. 

\textbf{Communication constraints.} To capture communication constraints, we assume that the channel from the agent to the server has a finite capacity of $B$ bits. Thus, at each time-step, the channel can transmit without error one of $2^B$ symbols denoted by $\sigma \in \Sigma$, where $\vert \Sigma \vert = 2^B$. As explained and motivated in the introduction,  we impose an additional information constraint that the agent can only transmit encoded estimates of the unknown model parameter $\theta_*$, but not the rewards themselves. We note here that encoding a high-dimensional model estimate is much more challenging than encoding a scalar-valued reward.  

In section \ref{sec:guarantees}, we will establish that with $B=O(d)$ bits, one can ensure that  $R_T = \tilde{O}\left(d\sqrt{T}\right)$. Arriving at this result is however  quite  non-trivial, and requires overcoming certain key technical challenges that we outline next.

\textbf{Challenges.} In the standard linear stochastic bandit formulation, the chief difficulty lies in taking decisions that incur low regret despite statistical uncertainty concerning the unknown parameter $\theta_*$. In our setting, such uncertainty is accentuated by the loss of information incurred over the finite-capacity channel. Unless the server explicitly accounts for this additional source of error in its decision-making process, it can end up taking sub-optimal actions that generate low rewards. Moreover, since our problem is of an inherently sequential nature, the effect of ``poor" actions coupled with channel-induced errors can pile up over time, resulting in the agent-server pair suffering linear regret. The above discussion highlights the challenge in decision-making.

In terms of communication, one natural idea to encode the parameter $\theta_*$ could be to uniformly quantize the set $\Theta$ at each time-step, since $\theta_* \in \Theta$. To ensure that the sum of the instantaneous encoding errors do not grow linearly with the horizon $T$, such errors need to scale inversely with $T$. However, to achieve such a precision with a  \textit{non-adaptive} encoding scheme that always encodes the entire set $\Theta$, the channel capacity $B$, in turn, needs to scale with $T$. 
This is highly undesirable since the horizon-length $T$ can be arbitrarily long. To sum up, the design of a joint encoding-decoding  and decision-making strategy that achieves order-optimal regret with a horizon-independent channel capacity is not at all obvious a priori. Nonetheless, we will develop such a strategy in the next section. For now, we lay down certain standard technical assumptions that will be made throughout the paper. 

\begin{assumption} \label{ass:actions}
The following hold: \begin{enumerate}
    \item[(i)] $\max_{t\in [T]} \sup_{a,b \in \mathcal{A}_t} \langle \theta_*, a-b \rangle \leq 1.$ 
    \item[(ii)] ${\Vert a \Vert}_2 \leq L$, $\forall a \in  \bigcup_{t=1}^{T} \mathcal{A}_t$. \item[(iii)] At each time-step $t\in[T]$, the decision set $\mathcal{A}_t$ contains the unit sphere $\mathbb{S}^{d-1}$.
\end{enumerate}
\end{assumption}
While assumptions (i) and (ii) are typical in the literature on linear stochastic bandits \cite{tor}, assumption (iii) is also quite standard and has been used in various different contexts \cite{amani,yangICLR}. Without loss of generality, we assume that $L \geq 1$; furthermore, we assume that the horizon is long-enough relative to the dimension of the model:  $T \geq d^2$. 

\section{Information-Constrained Optimism in the Face of Uncertainty}
\label{sec:lin_algo}

In this section, we will develop our proposed algorithm (Algorithm \ref{algo:ICLUCB}) called Information Constrained \texttt{LinUCB}  (\texttt{IC-LinUCB}) that comprises of two phases. Phase I is a pure exploration phase where the server picks i.i.d. actions from the uniform distribution over the unit sphere; such actions are feasible owing to Assumption \ref{ass:actions}-(iii). During this phase which lasts for $\bar{T}+1$ time-steps, the only transmission from the agent to the server takes place at time-step $\bar{T}+1$. The purpose of the pure exploration phase and the choice of the parameter $\bar{T}$ will be explained shortly. During each time-step of Phase II, the agent employs an \textit{adaptive} encoding strategy (outlined in Algorithm~\ref{algo:Encoder}) to transmit information about the unknown parameter $\theta_*$ to the server. Based on this information, the server takes decisions by constructing an ``inflated" confidence that accounts for encoding errors. We now describe in detail the two key ingredients of $\texttt{IC-LinUCB}$: (i) the adaptive encoding strategy at the agent, and (ii) the decision-making rule at the server.

$\bullet$ \textbf{Adaptive Encoding at Agent.} To describe the encoder, we will require the notion of an $\epsilon$-net \cite{versh}. 

\begin{definition} (\textbf{$\epsilon$-net}). Consider a subset $\mathcal{K} \subset \mathbb{R}^{d}$ and let $\epsilon > 0$. A subset $\mathcal{N} \subseteq \mathcal{K}$ is called an $\epsilon$-net of $\mathcal{K}$ if every point in $\mathcal{K}$ is within a distance of $\epsilon$ of some point of $\mathcal{N}$, i.e., 
$$ \forall x \in \mathcal{K}, \exists x_0 \in \mathcal{N}: {\Vert x-x_0 \Vert}_2 \leq \epsilon. $$
Equivalently, $\mathcal{N}$ is an $\epsilon$-net of $\mathcal{K}$ if and only if $\mathcal{K}$ can be covered by balls with centers in $\mathcal{N}$ and radii $\epsilon$. 
\end{definition}
Next, consider the least-squares estimate $\hat{\theta}^{(a)}_t$ maintained by the agent:
\begin{equation}
\hat{\theta}^{(a)}_t=V^{-1}_t \sum_{s=1}^{t}a_{s}y_{s}, \hspace{2mm} \textrm{where} \hspace{2mm} 
    V_t = \lambda I_d + \sum_{s=1}^{t}a_s a'_s
\label{eqn:least_square}
\end{equation}
is the covariance matrix at time-step $t$.\footnote{We use $x'$ to denote the transpose of a vector $x$.}  Here,  $\lambda >0$ is a scalar regularization parameter. Let $\hat{\theta}^{(s)}_t$ be the estimate of $\theta_*$ maintained by the server; $\hat{\theta}^{(s)}_t$ is initialized from any arbitrary vector in $\Theta$ at time-step $\bar{T}+1$. The choice of this initial vector is known to both the agent and the server. 

\textbf{Main Ideas.} The key ideas guiding our encoding strategy are as follows. Once the agent has acquired  sufficiently many observations, the gap $\hat{\theta}^{(a)}_{t}-\hat{\theta}^{(a)}_{t-1}$ between successive estimates will start shrinking due to the pure exploration phase; see Remark \ref{rem:Pureexp}.  Thus, at this stage, if the gap  $\hat{\theta}^{(a)}_{t-1} - \hat{\theta}^{(s)}_{t-1}$ is not too large, then the gap $e_t=\hat{\theta}^{(a)}_{t} - \hat{\theta}^{(s)}_{t-1}$ should not be too large either. In other words, eventually, a new observation $y_t$ will not cause the agent's estimate of $\theta_*$ to deviate drastically from the estimate of $\theta_*$ held by the server. Intuitively, it thus makes sense to encode and transmit only the \textit{new information} about $\theta_*$ contained in $y_t$, i.e., the ``innovation" signal $e_t$ (as opposed to encoding $\hat{\theta}^{(a)}_{t}$). However, given the stochastic nature of our setup, $e_t$ is a random variable. Thus, encoding $e_t$ poses the technical hurdle of characterizing the region containing $e_t$ with high probability. To this end, in Lemma \ref{lemma:enc_dec} of Section \ref{sec:guarantees}, we establish that with high probability, $\forall t \geq \bar{T}+1$, $e_t \in \mathcal{B}_d(0,p_t)$, where $p_t$ is the radius of the ball containing the innovation $e_t$. Our encoding strategy is adaptive since it requires dynamically updating the radius $p_t$ (as per Eq. \eqref{eqn:quantizer_eqs}) based on statistical concentration bounds specific to our problem.
\begin{figure}[t]
\centering
  \includegraphics[width=0.75\linewidth]{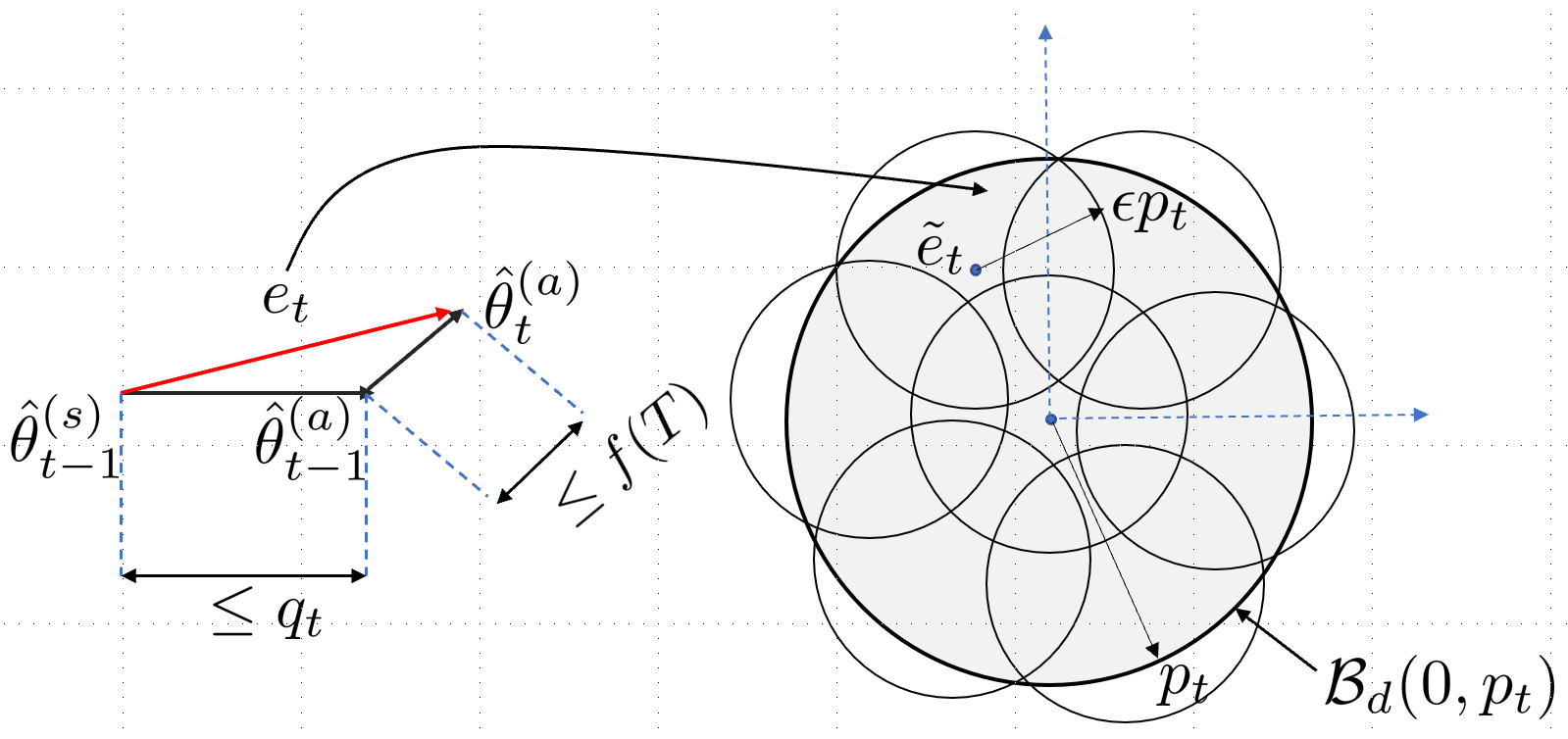}
  \caption{Illustration of the  encoding technique in Algorithm   \ref{algo:Encoder}. The agent computes the innovation signal $e_t$ that belongs to $\mathcal{B}_d(0,p_t)$ with high probability. An $\epsilon p_t$-net of $\mathcal{B}_d(0,p_t)$ is constructed, and the center $\tilde{e}_t$ of the ball containing $e_t$ is decoded by the server.}
  \label{fig:covering}
\end{figure} 

\begin{algorithm}[H]
\caption{Adaptive Encoding at the Agent}
\label{algo:Encoder}  
 \begin{algorithmic}[1] 
\Statex \hspace{-5mm} \textbf{Input Parameters:} $\hat{\theta}^{(s)}_{\bar{T}}$ is any arbitrary vector in $\Theta$; $q_{\bar{T}}=10M$; and $f(T)=\frac{3}{5L}\sqrt{\frac{\beta_T}{T\log(dLT)}}$.
\For {$t\in \{\bar{T}+1, \ldots, T\}$}
\State Observe $y_t$; compute $\hat{\theta}^{(a)}_t=V^{-1}_t \sum_{s=1}^{t}a_{s}y_{s}$ and innovation $e_t=\hat{\theta}^{(a)}_t-\hat{\theta}^{(s)}_{t-1}$. 
\State Encode $e_t$ by constructing an $\epsilon p_t$-net of $\mathcal{B}_d(0,p_t)$, where
\begin{equation}
  q_{t}=\epsilon \left( q_{t-1} + f(T)\right); \hspace{2mm}  p_t=q_t+f(T).
\label{eqn:quantizer_eqs}
\end{equation}
\EndFor
\end{algorithmic}
 \end{algorithm}
 
\textbf{Summary of Encoding Strategy.} The overall encoding technique in Algorithm \ref{algo:Encoder} can be summarized as follows. At each time-step $t \geq \bar{T}+1$, the agent observes $y_t$, computes $\hat{\theta}^{(a)}_t$ as per \eqref{eqn:least_square}, and then evaluates the innovation signal $e_t=\hat{\theta}^{(a)}_t-\hat{\theta}^{(s)}_{t-1}$. Given that $e_t \in \mathcal{B}_d(0,p_t)$ with high probability (as justified by Lemma \ref{lemma:enc_dec}),  the region  $\mathcal{B}_d(0,p_t)$ is covered by balls of radius $\epsilon p_t$, where $\epsilon\in (0,1)$ is a pre-decided constant, i.e., the agent constructs an $\epsilon p_t$-net of $\mathcal{B}_d(0,p_t)$.\footnote{For a discussion on constructing such coverings, see \cite{dumer,verger} and the references therein.} The agent then determines the ball $e_t$ falls into, and transmits the symbol $\sigma\in\Sigma$ corresponding to that ball.\footnote{In case $e_t$ lands on the boundary of more than one ball, it is assigned the label/symbol of any one of those balls based on a fixed priority rule.} If $e_t$ falls outside $\mathcal{B}_d(0,p_t)$, the agent transmits a special symbol to indicate an overflow. We succinctly represent the entire operation described above by a dynamic encoder map $\mathcal{E}_t$ that takes as input $e_t$  and generates as output the symbolic encoding $\sigma_t \in \Sigma$ that is transmitted to the server.

\textbf{Decoding at Server.} As a basic requirement for correct decoding, we assume that the server is aware of the encoding operation at the agent. Moreover, note that the sequences $\{p_t\}$ and $\{q_t\}$  defined in Eq. \eqref{eqn:quantizer_eqs} are deterministic, and can be computed by the server at its end. Thus, at any time-step $t\geq \bar{T}+1$, the server is aware of the region $\mathcal{B}_d(0,p_t)$ being encoded. Upon receiving $\sigma_t$, the server can thus correctly determine the center $\tilde{e}_t$ of the ball containing  $e_t$. We represent the above decoding operation at time $t$ by the decoder map $\mathcal{D}_t$ that takes as input $\sigma_t$ and outputs $\tilde{e}_t$. Having decoded the innovation signal, the server computes an estimate $\hat{\theta}^{(s)}_t$ of $\theta_*$ as per line 7 of Algo. \ref{algo:ICLUCB}. The agent computes $\hat{\theta}^{(s)}_t$ on its end as well in order to evaluate the innovation signal at time $t+1$; see line 2 of Algo. \ref{algo:Encoder}. This concludes the description of the encoding-decoding operation, a pictorial illustration of which is presented in Figure \ref{fig:covering}.

Till now, we have only described how to transit information about $\theta_*$ from the agent to the server over a finite-capacity channel. However, the key question that remains unanswered is the following: \textit{How should the server take decisions that yield low cumulative regret while accounting for the additional uncertainty introduced by the channel?} We now turn to answering this question. 

\begin{algorithm}[t]
\caption{Information Constrained \texttt{LinUCB}  (\texttt{IC-LinUCB})}
\label{algo:ICLUCB}  
 \begin{algorithmic}[1] 
\Statex \hspace{-5mm} \textbf{Input Parameters:} $\bar{T}=\ceil{10 L^2d\sqrt{T}\log(dLT)}$.
\Statex \hspace{-5mm} \textbf{Phase I:} \textit{Pure Exploration}
\For {$t\in \{1, \ldots, \bar{T}+1\}$} 
\State Server plays $a_t \sim  \textrm{Unif}(\mathbb{S}^{d-1})$. 
\State Agent receives reward $y_t$ as per \eqref{eqn:obs_model} and computes estimate  $\hat{\theta}^{(a)}_t=V^{-1}_t \sum_{s=1}^{t}a_{s}y_{s}$.
\EndFor
\State Agent encodes $e_{\bar{T}+1}=\hat{\theta}^{(a)}_{\bar{T}+1}-\hat{\theta}^{(s)}_{\bar{T}}$ as per Algo. \ref{algo:Encoder}, and transmits $\sigma_{\bar{T}+1}=\mathcal{E}_{\bar{T}+1}\left(e_{\bar{T}+1}\right)$. 
\Statex \hrulefill
\Statex \hspace{-5mm} \textbf{Phase II:} \textit{Information-Constrained Exploration-Exploitation}
 \For {$t\in \{\bar{T}+2, \ldots, T\}$}
\State Server decodes $\tilde{e}_{t-1} =\mathcal{D}_{t-1}(\sigma_{t-1})$, and generates $\hat{\theta}^{(s)}_{t-1}=\hat{\theta}^{(s)}_{t-2}+\tilde{e}_{t-1}$. 
\State Server constructs confidence ellipsoid: 
\begin{equation}
    \mathcal{C}^{(s)}_t = \{\theta\in\mathbb{R}^d: {\Vert \theta - \hat{\theta}^{(s)}_{t-1} \Vert}_{V_{t-1}} \leq \sqrt{\beta_T} + \textcolor{black}{\left(\sqrt{\lambda+(t-1)L^2}\right) q_t\}},
\label{eqn:dec_conf_int}
\end{equation}
where  $V_{t-1}$, $q_t$, and $\sqrt{\beta_T}$ are given by \eqref{eqn:least_square}, \eqref{eqn:quantizer_eqs}, and \eqref{eqn:conf_radius}, respectively. 
\State Server plays action $a_t = \argmax_{a\in\mathcal{A}_t} \max_{\theta\in\mathcal{C}^{(s)}_t} \langle \theta, a \rangle$. 
\State Agent receives reward $y_t$ as per \eqref{eqn:obs_model} and computes estimate  $\hat{\theta}^{(a)}_t=V^{-1}_t \sum_{s=1}^{t}a_{s}y_{s}$.
\State Agent encodes the innovation $e_t=\hat{\theta}^{(a)}_t-\hat{\theta}^{(s)}_{t-1}$ as per Algo. \ref{algo:Encoder}, and transmits $\sigma_t = \mathcal{E}_t(e_t)$. 
\EndFor
\end{algorithmic}
 \end{algorithm}

$\bullet$ \textbf{Decision-Making at the Server.} When there is no loss of information over the channel, i.e., when $\hat{\theta}^{(s)}_t=\hat{\theta}^{(a)}_t$, the celebrated \texttt{LinUCB} algorithm relies on the principle of \textit{optimism in the face of uncertainty.} Specifically, at each time-step, an ellipsoidal confidence set is constructed that contains $\theta_*$ with high-probability. The learner then acts optimistically by playing an action that yields the highest reward over all possible values of $\theta$ in the confidence set. While  our approach builds on the same high-level principle, it  relies crucially on the construction of a new ``inflated" ellipsoidal confidence set: 
\begin{equation}
    \mathcal{C}^{(s)}_t = \{\theta\in\mathbb{R}^d: {\Vert \theta - \hat{\theta}^{(s)}_{t-1} \Vert}_{V_{t-1}} \leq \sqrt{\beta_T} + \textcolor{myblue}{\left(\sqrt{\lambda+(t-1)L^2}\right) q_t\}},
\end{equation}
where 
\begin{equation}
    \sqrt{\beta_T} = \sqrt{\lambda}M + \sqrt{2\log\left(\frac{1}{\delta}\right)+d\log\left(\frac{d\lambda+TL^2}{d\lambda}\right)}.
\label{eqn:conf_radius}
\end{equation}
Here, recall that $M \geq 1$ is such that for each $\theta\in\Theta$, ${\Vert \theta \Vert}_2 \leq M$; $\delta\in(0,1)$ is a constant to be picked later.  Notably, when $\hat{\theta}^{(s)}_{t-1}=\hat{\theta}^{(a)}_{t-1}$, and $q_t=0$, $\mathcal{C}^{(s)}_t$ reduces to the confidence set in \texttt{LinUCB}. The inflation in the radius of the confidence set (relative to \texttt{LinUCB}) carefully accounts for the quantization errors resulting from the finite capacity of the channel. Our main technical contribution in this regard is to establish that $\forall t\geq \bar{T}+2$, $\theta_* \in \mathcal{C}^{(s)}_t$ with high probability; see Lemma \ref{lemma:dec_conf_region} in Section \ref{sec:guarantees}. This result, in turn,  justifies the optimistic decision-making rule of \texttt{IC-LinUCB} in line 9 of Algorithm  \ref{algo:ICLUCB}. During the pure exploration phase, the server simply samples actions independently from the uniform distribution over $\mathbb{S}^{d-1}$, i.e., $a_t \sim  \textrm{Unif}(\mathbb{S}^{d-1}), \forall t\in[\bar{T}+1]$.\footnote{To be more precise, a random variable $Z$ is uniformly distributed on $\mathbb{S}^{d-1}$ if, for every Borel subset $\mathcal{K} \subset \mathbb{S}^{d-1}$, the probability $\mathbb{P}(Z \in \mathcal{K})$ equals the ratio of the $(d-1)$-dimensional areas of $\mathcal{K}$ and $\mathbb{S}^{d-1}$.} At every time-step $t\in[T]$, the action $a_t$ decided upon by the server is passed down to the agent without any loss of information.

We summarize below the essential steps of our approach. 
\begin{itemize}
\item \textbf{Bound Gap between Successive Model Estimates.} Bound the gap $\hat{\theta}^{(a)}_{t}-\hat{\theta}^{(a)}_{t-1}$ between successive estimates of $\theta_*$ (at the agent), and argue that this gap is eventually small due to the pure exploration phase; see Remark \ref{rem:Pureexp}. 

\item \textbf{Characterize Encoding Region.} Using the bound from the above step, encode the innovation signal $e_t=\hat{\theta}^{(a)}_{t}-\hat{\theta}^{(s)}_{t-1}$ by characterizing the region that contains $e_t$ with high-probability.

\item \textbf{Construct Inflated Confidence Set.} Construct a confidence set that (i) is centered at the server's estimate of $\theta_*$, namely $\hat{\theta}^{(s)}_t$, (ii) accounts for encoding errors, and (iii) contains $\theta_*$ with high-probability. Play actions optimistically w.r.t. this confidence set. 
\end{itemize}

In Section \ref{sec:gen_lin}, we will demonstrate that the above ideas are applicable to parameterized observation models well beyond the specific linear model considered in Eq. \eqref{eqn:obs_model}. As such, our approach provides a general recipe for decision-making under information constraints. Before moving on to the performance guarantees of \texttt{IC-LinUCB}, a couple of important remarks are in order. 

\begin{remark} \label{rem:Pureexp}  (\textbf{Need for Pure Exploration Phase}) Note that $\forall t\geq \bar{T}+1$, $\lambda_{min}(V_t) \geq \lambda_{min}(V_{\bar{T}})$. Moreover, for $\bar{T}$ as chosen in Algorithm  \ref{algo:ICLUCB}, the pure exploration phase ensures that $\lambda_{min}(V_{\bar{T}})$ is bounded below by a linear function of $\bar{T}$.  The above two facts collectively imply that the gap between consecutive estimates of $\theta_*$ (at the agent) is eventually small - a key requirement for our encoding strategy. This point is made precise in Lemma \ref{lemma:iterate_bound} of Appendix \ref{app:ThmICLinUCB}. The duration $\bar{T}$ of the pure exploration phase is chosen carefully to ensure that the term $\left(\sqrt{\lambda+(t-1)L^2}\right) q_t$ in $\mathcal{C}^{(s)}_t$ is eventually $O\left(\sqrt{\beta_T}\right)$. In other words, such a choice of $\bar{T}$ enables us to preserve order-optimal regret despite taking actions based on an inflated confidence set. 
\end{remark}

\begin{remark} (\textbf{Horizon-Independent Channel Capacity}) In our encoding technique, although the radii of the balls used to cover the encoding region shrink over time, the radius of the ball being encoded shrinks commensurately. This is precisely what enables us to achieve order-optimal regret with a channel capacity $B=O(d)$ bits that is independent of the time-horizon $T$. Roughly speaking, the main intuition here is that to achieve the same level of encoding precision, it takes fewer bits to encode the innovation $e_t$ as compared to the model $\theta_*$. This is because while $e_t$ resides (with high probability) in a ball of progressively shrinking radius, $\theta_*$ belongs to the set $\Theta$ of fixed radius.  
\end{remark}
\section{Analysis of the \texttt{IC-LinUCB} Algorithm}
\label{sec:guarantees}
Our main result concerning the performance of the \texttt{IC-LinUCB} algorithm is as follows.

\begin{theorem} (\textbf{Regret of \texttt{IC-LinUCB}}) \label{thm:ICLinUCB}
Suppose Assumption \ref{ass:actions} holds, and let the channel capacity satisfy $B \geq 6d$. Then, with $\epsilon=1/2$ and $\delta=1/T$ in Algorithm \ref{algo:ICLUCB}, the \texttt{IC-LinUCB} algorithm guarantees:
\begin{equation}
    R_T = O\left(L^2 d \sqrt{T} \log(dLT)\right) = \tilde{O}\left(d\sqrt{T}\right).
\end{equation}
\end{theorem}

\textbf{Discussion.} We note that for the \texttt{IC-LinUCB} algorithm, the dependence of the regret on $d$ and $T$ exactly matches that of \texttt{LinUCB}. Thus, our work is the first to establish that with a horizon-independent channel capacity of $O(d)$ bits, one can achieve the same performance as when the channel has infinite capacity. Thus, Theorem \ref{thm:ICLinUCB} can be seen as an extension of the results in \cite{dani,abbasi} to the communication-constrained setting of interest in this work.
Interestingly, \cite{mayekar} recently showed that for stochastic optimization with $d$-dimensional quantized gradients, a bit-rate of $\Omega(d)$ is necessary for achieving the optimal convergence rate of $O(1/\sqrt{T})$, where $T$ is the number of iterations. We conjecture that to achieve order-optimal regret, a similar lower-bound of $B=\Omega(d)$ will hold for our setup as well; we leave verifying this as future work. 

We prove Theorem \ref{thm:ICLinUCB} in Appendix \ref{app:ThmICLinUCB}. In what follows, we briefly outline the key technical steps in the proof.

\textbf{Outline of the proof.} We start by constructing an appropriate ``clean" event $\mathcal{G}$ of measure at least $1-5/T$, and condition on this event throughout the subsequent analysis. There are three main steps in the proof of Theorem \ref{thm:ICLinUCB}, and we describe them below. 

$\bullet$ \textbf{Step 1.} On the clean event $\mathcal{G}$, we argue that the gap between successive model estimates at the agent is eventually small. More precisely, in Lemma  \ref{lemma:iterate_bound} of Appendix \ref{app:ThmICLinUCB}, we establish that 
$$ {\Vert \hat{\theta}^{(a)}_{t+1}-\hat{\theta}^{(a)}_{t}\Vert}_2 \leq f(T), \forall t \geq \bar{T}, $$
where $f(T)$ is as defined in the input parameters of Algorithm \ref{algo:Encoder}. The proof of Lemma \ref{lemma:iterate_bound} in turn relies on the fact that with high probability,
$$ \lambda_{\min}(V_t) \geq 5 L^2 \sqrt{T} \log(dLT), \forall t\geq \bar{T}. $$
The above claim is established in Lemma \ref{lemma:eigen_cov} of Appendix \ref{app:ThmICLinUCB} by appealing to the Matrix Bernstein inequality \cite[Theorem 5.4.1]{versh}. 
\vspace{2mm}

$\bullet$ \textbf{Step 2.} The next key result justifies the encoding strategy in Algorithm \ref{algo:Encoder}.\footnote{We note that results of a similar conceptual flavor are established in \cite{tatikonda} and \cite{kostina} in the context of stabilization of an LTI system, and optimization, respectively. While the results in these papers pertain to deterministic settings, Lemma \ref{lemma:enc_dec} carefully exploits statistical concentration bounds specific to the stochastic process we study.} 

\begin{lemma} (\textbf{Encoding Region}) The following is true with probability at least $1-5/T$:
$$ e_t \in \mathcal{B}_{d}(0,p_t), \forall t \in \{\bar{T}+1, \ldots, T\},$$
where $e_t$ is the innovation in line 2 of Algorithm \ref{algo:Encoder}, and $p_t$ is as defined in Eq.  \eqref{eqn:quantizer_eqs}.
\label{lemma:enc_dec}
\end{lemma}

The above result tells us that the innovation random variable $e_t$ always falls within the desired encoding region on the event $\mathcal{G}$, i.e., with high probability, there is never any overflow. It is easy to argue that $B=O(d)$ bits suffice to construct an $\epsilon p_t$ net of $\mathcal{B}_d(0,p_t)$. 
\vspace{2mm}

$\bullet$ \textbf{Step 3.} It remains to justify the choice of the confidence set $\mathcal{C}^{(s)}_t$ in Eq. \eqref{eqn:dec_conf_int}. This is achieved in the following lemma. 

\begin{lemma} (\textbf{Confidence Region at Server}) With probability at least $1-5/T$,  the following is true: $\theta_* \in \mathcal{C}^{(s)}_t, \forall t \in \{\bar{T}+2, \ldots, T\}$, where $\mathcal{C}^{(s)}_t$ is the confidence set defined in Eq. \eqref{eqn:dec_conf_int}. 
Moreover, $\forall t \geq \bar{T}+\tilde{T}$, we have 
\begin{equation}
   \left(\sqrt{\lambda+(t-1)L^2}\right) q_t \leq 4 \sqrt{\frac{\beta_T}{\log(dLT)}},
\end{equation}
where 
\begin{equation}
    \tilde{T} =  \ceil*{\frac{\log\left(\frac{10M}{f(T)}\right)}{\log(2)}} \vee 2 = O\left(\log(dLT)\right).
\end{equation}
\label{lemma:dec_conf_region}
\end{lemma}

The above result implies that the inflated confidence set (that accounts for encoding errors) eventually contains the true parameter $\theta_*$ with high probability. At the same time, the quantization error $q_t$ decays fast enough to ensure that the radius of the confidence set is eventually $O(\sqrt{\beta_T})$ - exactly as in the \texttt{LinUCB} algorithm. In other words, our approach ensures that the impact of the quantization error on decision-making vanishes over time.

Equipped with Lemma's \ref{lemma:enc_dec} and \ref{lemma:dec_conf_region}, the rest of the proof of Theorem \ref{thm:ICLinUCB} follows fairly standard arguments. We fill in these detailed arguments in Appendix \ref{app:ThmICLinUCB}. 
\newpage
\section{Extension to Generalized Linear Models}
\label{sec:gen_lin}
The main goal of this section is to demonstrate the generality of the algorithmic approach developed in Section \ref{sec:lin_algo}. To do so, we will now consider an observation model where the rewards are no longer assumed to be linear functions of the parameter. Rather, they satisfy the following relationship \cite{GLM1,GLM2}:
\begin{equation}
    y_t=\mu\left(\langle \theta^*, a_t \rangle \right)+\eta_t,
\label{eqn:gen_obs_model}
\end{equation}
where $\mu:\mathbb{R} \rightarrow \mathbb{R}$ is a continuously differentiable function typically referred to as the (inverse) \textit{link} function, and $\{\eta_t\}$ is a 1-subgaussian noise process as before. Our goal now will be to control the following notion of regret: 
\begin{equation}
\bar{R}_T=\mathbb{E} \left[\sum_{t=1}^{T} \left( \max_{a\in\mathcal{A}_t} \mu(\langle \theta_*, a \rangle) - \mu(\langle \theta_*, a_t \rangle) \right)  \right].
\label{eqn:GLM_regret}
\end{equation}
In \cite{GLM1}, the authors noted the following main difficulty in analyzing the above non-linear model relative to the linear setting: for the Generalized Linear Model (GLM) in \eqref{eqn:obs_model}, the relevant confidence regions may have a more complicated geometry in the parameter space than simple ellipsoids. To overcome this challenge, the \texttt{GLM-UCB} algorithm is developed in \cite{GLM1} by focusing on the reward space. 

For our setting, however, recall that the agent is not allowed to transmit the reward $y_t$ (neither directly, nor in an encoded form) to the server, thereby further adding to the technical complexity. This raises the following questions. (i) How should the agent encode an estimate of $\theta_*$ based on the GLM in Eq. \eqref{eqn:gen_obs_model}? (ii) How should the server design the  confidence region for decision-making? (iii) Will $B=O(d)$ bits continue to suffice for achieving order-optimal regret? In what follows, we will answer the last question in the affirmative by appropriately adapting the general recipe outlined at the end of Section \ref{sec:lin_algo}. 

To get started, we make the following standard assumption on the inverse link function \cite{GLM1,GLM2}.

\begin{assumption} \label{ass:link} The function $\mu:\mathbb{R} \rightarrow \mathbb{R}$ is continuously differentiable, Lipschitz with constant $k_2$, and such that $$k_1=\min\{1,\inf_{\theta\in\Theta, a \in \bigcup_{t\in[T]}\mathcal{A}_t} \dot{\mu}(\langle \theta, a \rangle) \} > 0.$$
Here, $\dot{\mu}(\cdot)$ is used to represent the derivative of $\mu(\cdot)$.
\end{assumption}

Without loss of generality, we will assume $k_2 \geq 1$. Next, let us define 
\begin{equation}
    g_t(\theta)=\lambda \theta + \sum_{s=1}^{t}\mu(\langle \theta, a_s \rangle) a_s.
\label{eqn:g_func}
\end{equation}

We are now ready to describe the Information-Constrained \texttt{GLMUCB} algorithm  (\texttt{IC-GLMUCB}). 

\noindent \textbf{Description of \texttt{IC-GLMUCB}}. The \texttt{IC-GLMUCB} algorithm shares the same structure as \texttt{IC-LinUCB}, but has some crucial algorithmic differences that we next outline. First, instead of computing the least-squares estimate as in Eq. \eqref{eqn:least_square}, the agent now computes an estimate $\hat{\theta}^{(a)}_t$ of $\theta^*$ by solving the following equation:
\begin{equation}
    g_t(\hat{\theta}^{(a)}_t) = \sum_{s=1}^{t} y_s a_s.
\label{eqn:non_lin_LS}
\end{equation}
It then encodes the innovation signal $e_t=\hat{\theta}^{(a)}_t-\hat{\theta}^{(s)}_{t-1}$ exactly as in line 3 of Algo. \ref{algo:Encoder}, with $f(T)$ in Eq. \eqref{eqn:quantizer_eqs} replaced by 
\begin{equation}
    \bar{f}(T) = \frac{3}{5L} \sqrt{ \frac{\beta_T}{k_1 k_2 T\log(dLT)}},
\end{equation}
and $q_{\bar{T}}$ set to $\left(4+6/\sqrt{k_1 k_2}\right)M$. The server performs decoding exactly as in line 7 of Algorithm \ref{algo:ICLUCB}, and computes $\hat{\theta}^{(s)}_t$. The main distinction relative to \texttt{IC-LinUCB} is in the construction of the confidence region at the server. The server constructs the following confidence set:
\begin{equation}
    \mathcal{\bar{C}}^{(s)}_t =  \{\theta\in \Theta: \mathcal{H}_{t-1}(\theta) \leq \sqrt{\beta_T} + k_2\textcolor{black}{\left(\sqrt{\lambda+(t-1)L^2}\right) q_t\}}, \hspace{1mm} \textrm{where} \hspace{1mm}  \mathcal{H}_t(\theta) \triangleq {\Vert g_t(\theta)-g_t(\hat{\theta}^{(s)}_t) \Vert}_{V^{-1}_t}.
\label{eqn:GLM_conf_region}
\end{equation}
The server then acts optimistically w.r.t. the above confidence set:
\begin{equation}
    a_t = \argmax_{a\in\mathcal{A}_t} \max_{\theta\in\mathcal{\bar{C}}^{(s)}_t} \langle \theta, a \rangle, \forall t \in \{\bar{T}+2,\ldots,T\}, 
\end{equation}
with $\bar{T}$ now set to
$$ \bar{T}=\ceil*{10 \frac{k_2}{k_1}  L^2d\sqrt{T}\log(dLT)}. $$

For completeness, we provide the detailed steps of the \texttt{IC-GLMUCB} algorithm in Appendix \ref{app:GLM_analysis}. To analyze the performance of the \texttt{IC-GLMUCB} algorithm, we will make the following standard assumption that is in the same spirit as Assumption \ref{ass:actions}(i):
$$ \sup_{a,b \in \bigcup_{t=1}^{T} \mathcal{A}_t}  \mu(\langle \theta_*,a \rangle) - \mu(\langle \theta_*,b \rangle) \leq 1. $$

We now state the main result of this section. 
\begin{theorem} (\textbf{Regret of \texttt{IC-GLMUCB}}) \label{thm:ICGLMUCB}
Suppose Assumptions \ref{ass:actions}(ii)-(iii) hold along with Assumption \ref{ass:link}. Moreover, let the channel capacity satisfy $B \geq 6d$. Then, the \texttt{IC-GLMUCB} algorithm guarantees:
\begin{equation}
    \bar{R}_T = O\left( { \left(\frac{k_2}{k_1}\right)}^{3/2} L^2 d \sqrt{T} \log(dLT) \right) = \tilde{O}\left(d\sqrt{T}\right).
\end{equation}
\end{theorem}

For the proof of Theorem \ref{thm:ICGLMUCB}, see Appendix \ref{app:GLM_analysis}. 

\textbf{Discussion.} The above result significantly generalizes Theorem \ref{thm:ICLinUCB}, and reveals that the \texttt{IC-GLMUCB} algorithm yields the same regret bounds as \texttt{IC-LinUCB} (up to constants) under identical requirements on the channel capacity as before. Thus, the main takeaway here is that our approach can readily accommodate general non-linear observation models as well (albeit with a slightly more involved analysis). 
\newpage
\section{One Bit Channel Capacity is Sufficient for the Multi-Armed Bandit Problem}
\label{sec:unstructred}
In Sections \ref{sec:guarantees} and \ref{sec:gen_lin}, we have seen that for a $d$-dimensional model, $O(d)$ bits suffice to achieve order-optimal regret. In this section, we investigate whether one can achieve similar order-optimal regret bounds with fewer bits when the set of feasible actions has additional structure. We will show that this is indeed the case for a particular setting of interest when $\mathcal{A}_t =\{e_1, \ldots, e_d\}, \forall t\in [T]$, where ${(e_i)}_i$ are the standard orthonormal unit vectors. This setting represents the popular unstructured multi-armed bandit problem with a finite number of arms.  Our main insight is the following: playing action/arm $i$ only reveals information about the $i$-th component of $\theta_*$, denoted by $\theta_i$, and hence, when the $i$-th action is played, it makes sense for the agent to encode and transmit the innovation related to only $\theta_i$. In other words, the above intuition suggests that encoding a scalar innovation signal (as opposed to a $d$-dimensional innovation vector) should suffice for the specific setting under consideration. In what follows, we formalize this reasoning. 

To get started, let us note that the optimal action $a_*$ is the unit vector corresponding to the largest component of $\theta_*$. Without loss of generality, let this component be $\theta_1$, i.e., $\theta_1 = \max_{i\in[d]} \theta_i$. We thus have $a_*=e_1$. Let us denote by $\hat{\theta}^{(a)}_{i,k}$ (resp., $\hat{\theta}^{(s)}_{i,k}$) the estimate of $\theta_i$ at the agent (resp., at the server) after arm $i$ has been played $k$ times. We now develop an information-constrained variant of the celebrated upper confidence bound algorithm that we call \texttt{IC-UCB}.  

\noindent \textbf{Description of \texttt{IC-UCB}}. Let $\gamma=1/2^B$ where $B$ is the channel capacity, and define the following sequences for $k\geq 1$:
\begin{equation}
    p_{k+1}=\gamma p_k + 2 f_k; \hspace{2mm} q_{k}=\gamma p_k; \hspace{2mm} f_k= 2 \sqrt{\frac{\log T}{k}},
\label{eqn:UCB_quant_eq}
\end{equation}
where $p_1=m+f_1$, and $m \geq 1$ is such that $\max_{i\in[d]} |\theta_i| \leq m$. Suppose the action at time $t$ is $a_t=e_i$. The agent first updates its estimate of $\theta_i$:
\begin{equation}
    \hat{\theta}^{(a)}_{i,n_i(t)}= \frac{1}{n_i(t)} \sum_{k=1}^{n_i(t)} y_{i,k},
\label{eqn:run_sum}
\end{equation}
where $y_{i,k}$ is the agent's observation when the $i$-th arm is played the $k$-th time, and $n_i(t)$ is the number of times arm $i$ is played up to (and including) time-step $t$. It then computes the \textit{scalar} innovation $e_{i,n_i(t)}=\hat{\theta}^{(a)}_{i,n_i(t)}-\hat{\theta}^{(s)}_{i,n_i(t)-1}$, where $\hat{\theta}^{(s)}_{i,0}=0, \forall i\in[d]$. If $e_{i,n_i(t)}$ falls in the interval $\mathcal{Z}_{i,t}=[-p_{n_i(t)}, p_{n_i(t)}]$, then $\mathcal{Z}_{i,t}$ is partitioned uniformly into $2^B$ bins, and the symbol $\sigma_t \in \Sigma$ encoding the bin containing $e_{i,n_i(t)}$ is transmitted to the server. The server then decodes the center $\tilde{e}_{i,n_i(t)}$ of that bin, and computes $$\hat{\theta}^{(s)}_{i,n_i(t)}=\hat{\theta}^{(s)}_{i,n_i(t)-1}+\tilde{e}_{i,n_i(t)}.$$

If $e_{i,n_i(t)} \notin \mathcal{Z}_{i,t} $, then there is no transmission from the agent to the server. As for decision-making, each arm is first played once by the server. Subsequently, the action chosen by the server at time-step $t+1$ is the one that maximizes the following index:
\begin{equation}
    \texttt{IC-UCB}_i(t)=\hat{\theta}^{(s)}_{i,n_i(t)}+q_{n_i(t)}+f_{n_i(t)},
\label{eqn:index}
\end{equation}
where $q_{n_i(t)}$ and $f_{n_i(t)}$ are as in Eq. \eqref{eqn:UCB_quant_eq}. 

Let us define by $\Delta_i=\theta_1 - \theta_i$ the sub-optimality gap of arm $i$. To present our results in a clean way, we will focus on the particularly important case where the sub-optimality gaps are small: $\Delta_i \in (0,1], \forall i \in [d] \setminus\{1\}.$ Our results can be easily generalized to arbitrary values of the sub-optimality gaps. For the setting considered in this section, it is easy to verify that the regret $R_T$ in Eq. \eqref{eqn:regret} simplifies to
$$ R_T= \sum_{i=1}^{d}\Delta_i \mathbb{E}[n_i(T)]. $$ 
Our main result concerning the regret bound of \texttt{IC-UCB} is as follows.
\begin{theorem} 
\label{thm:ICUCB}
(\textbf{Regret of \texttt{IC-UCB}}) Suppose the channel capacity is at least $1$ bit, i.e., $B \geq 1$. The \texttt{IC-UCB} algorithm then guarantees:
$$ R_T \leq  5 \sum_{i=1}^{d} \Delta_i + \sum_{i=1}^{d} O\left(\frac{\log (mT)}{\Delta_i}\right). $$
\end{theorem}

We can also establish the following gap-independent bound.
\begin{theorem} 
\label{thm:gap_ind}
(\textbf{Gap-independent regret of \texttt{IC-UCB}}) Suppose the channel capacity satisfies $B \geq 1$. The \texttt{IC-UCB} algorithm then guarantees:
$$ R_T \leq  5 \sum_{i=1}^{d} \Delta_i + O\left(\sqrt{dT \log(mT)}\right). $$
\end{theorem}

\textbf{Discussion:} Our bounds above match those of the \texttt{UCB} algorithm, revealing that for the MAB problem, one can achieve order-optimal bounds with a bit-rate of \textit{just $1$ bit}. The main takeaway here is that when the action sets have more structure, there is hope for achieving optimal performance with a channel capacity of fewer than $O(d)$ bits. As future work, it would be interesting to see if one can draw similar conclusions for other types of common action sets. 

\section{Conclusion and Future Directions} 
We introduced and studied a new linear stochastic bandit problem subject to communication channel constraints. We developed a general algorithmic framework comprising of an adaptive compression mechanism, and a decision-making rule that explicitly accounts for encoding errors. We then showed how this framework leads to order-optimal regret bounds for (i) the linear bandit setting, (ii) the generalized linear bandit setting, and (iii) the MAB problem, with horizon-independent bit-rates that depend only on the dimension of the unknown model parameter. Our work opens up several interesting avenues of research; we discuss some of them below.

First, an immediate question that remains unresolved is whether $O(d)$ bits is indeed necessary for achieving order-optimal regret for the linear bandit problem. Answering this question would require establishing algorithm-independent lower bounds for our setup - this is quite non-trivial, and is the subject of our ongoing work. Second, while our focus was on minimizing the number of bits needed to achieve optimal regret, one may ask the following alternate question: Given a fixed communication budget, what is the best cumulative regret bound one can hope for? It would be particularly interesting to ascertain the minimal communication budget (as a function of the model parameter dimension) needed to achieve sublinear regret. Third, we plan to extend the ideas and results in our paper to more complex distributed/federated settings involving multiple agents. Finally, our goal is to explore whether the techniques developed in this paper are applicable to Markov decision processes in the context of reinforcement learning. 
\bibliographystyle{unsrt} 
\bibliography{refs}
\newpage
\appendix
\section{Supporting Technical Results}
\label{app:support}
In this section, we compile certain  results that will be used in our subsequent analysis. We start by recalling the Matrix Bernstein inequality that will be used by us for lower-bounding the smallest eigenvalue of the covariance matrix $V_t$. 

\begin{lemma} (\textbf{Matrix Bernstein})  \cite[Theorem 5.4.1]{versh}
Let $X_1, \ldots, X_N$ be independent zero-mean $d \times d$ symmetric random matrices, such that ${\Vert X_t \Vert}_2 \leq K$ almost surely, $\forall t \in [N]$. Then, for every $\epsilon > 0$, we have
\begin{equation}
    \mathbb{P}\left( {\norm[\bigg] { \sum_{t=1}^{N} X_t} }_2 \geq \epsilon \right) \leq 2d \exp\left(\frac{-\epsilon^2 / 2}{\sigma^2+K\epsilon/3} \right),
\label{eqn:mat_bern}
\end{equation}
where $\sigma^2={\norm[\bigg]  {\sum\limits_{t=1}^{N} \mathbb{E}[X^2_t]} }_2$ is the norm of the matrix variance of the sum. 
\label{lemma:mat_bern}
\end{lemma}

The next result is popularly used in the analysis of linear stochastic bandits. We adapt it based on our notation. 

\begin{lemma} (\textbf{Confidence Region}) \cite[Theorem 20.5]{tor} Given any $\delta \in (0,1)$, it holds that 
$$\mathbb{P}\left(\exists t \in [T]: \theta_* \notin \mathcal{C}_t\right) \leq \delta,$$ where 
\begin{equation}
 \mathcal{C}_t=\{\theta \in \mathbb{R}^d: {\Vert \hat{\theta}^{(a)}_{t-1} - \theta \Vert}_{V_{t-1}} \leq \sqrt{\beta_T}\},
 \label{eqn:enc_conf_region}
\end{equation}
and 
$$ \sqrt{\beta_T} = \sqrt{\lambda}M + \sqrt{2\log\left(\frac{1}{\delta}\right)+d\log\left(\frac{d\lambda+TL^2}{d\lambda}\right)}. $$
\label{lemma:enc_conf_region}
\end{lemma}

For reasoning about the capacity of the channel, we will require the concept of covering numbers. 
\begin{definition} (\textbf{Covering Numbers}) Let $\mathcal{K}$ be a subset of $\mathbb{R}^d$. The smallest cardinality of an $\epsilon$-net of $\mathcal{K}$ is called the covering number of $\mathcal{K}$ and is denoted by $\mathcal{N}(\mathcal{K},\epsilon)$. Equivalently, $\mathcal{N}(\mathcal{K},\epsilon)$ is the smallest number of closed balls with centers in $\mathcal{K}$ and radii $\epsilon$ whose union covers $\mathcal{K}$. 
\end{definition}

The following key result relates the covering number of a set $\mathcal{K}$ to its volume. 

\begin{lemma} \label{lemma:covers}  (\textbf{Covering Numbers and Volume}) \cite[Proposition 4.2.12]{versh} Let $\mathcal{K}$ be a subset of $\mathbb{R}^d$, and $\epsilon >0$. Then,
$$ \mathcal{N}(\mathcal{K},\epsilon) \leq \frac{\vert \left(\mathcal{K}\oplus(\epsilon/2)\mathcal{B}_d(0,1)\right)\vert}{\vert (\epsilon/2)\mathcal{B}_d(0,1)\vert}. $$
Here, we used $|\mathcal{K}|$ to represent the volume of a set $\mathcal{K}$, and $\mathcal{K}_1 \oplus \mathcal{K}_2$ to denote the Minkowski sum of two sets $\mathcal{K}_1, \mathcal{K}_2 \subset \mathbb{R}^d$.  
\end{lemma}
\newpage


\section{Proof of Theorem \ref{thm:ICLinUCB}}
\label{app:ThmICLinUCB}
In this section, we will develop the proof of Theorem \ref{thm:ICLinUCB}. In the process, we will restate some of the key lemmas from Section \ref{sec:guarantees} for the reader's convenience. We start with the following result that provides a lower-bound on the smallest eigenvalue of the covariance matrix $V_t$. 
\begin{lemma} \label{lemma:eigen_cov} 
Given any $\delta \in (0,1)$, we have that with probability at least $1-\delta$, 
\begin{equation}
    \lambda_{\min}(V_t) \geq \frac{\bar{T}}{2d}, \forall t\geq \bar{T},
\end{equation}
as long as 
$\bar{T} \geq \frac{28d}{3}  \log(2d/\delta)$. In particular, with $\delta = 2/T$, and $\bar{T}=\ceil{10L^2d\sqrt{T}\log(dLT)}$, we have that with probability at least $1-2/T$, 
\begin{equation}
    \lambda_{\min}(V_t) \geq 5 L^2 \sqrt{T} \log(dLT), \forall t\geq \bar{T}.
\label{eqn:min_eigen}
\end{equation}
\end{lemma}
\begin{proof}
The proof is an application of the Matrix Bernstein inequality in Lemma \ref{lemma:mat_bern}. To get started, recall that the pure exploration phase of the \texttt{IC-LinUCB} algorithm lasts for the first $\bar{T}+1$ time-steps, and that $a_t \sim \textrm{Unif}(\mathbb{S}^{d-1}), \forall t\in[\bar{T}+1]$. Define
\begin{equation}
    X_t \triangleq a_t a'_t - \bar{\Sigma}; \hspace{2mm} \textrm{and} \hspace{2mm}  \bar{\Sigma} \triangleq \frac{I_d}{d}. 
\end{equation}
From the choice of the actions during the pure exploration phase, we then immediately have 
\begin{equation}
\mathbb{E}[X_t] = 0_d; \hspace{2mm} \textrm{and} \hspace{2mm} \lambda_{\max}(X_t) \leq 1, \forall t \in [\bar{T}], 
\end{equation}
where we used $0_d$ to represent a $d \times d$ matrix with every entry equal to $0$. Since each $X_t$ is symmetric, we have ${\Vert X_t \Vert}_2 \leq 1, \forall t \in [\bar{T}].$ Now let $Y=\sum_{t=1}^{\bar{T}} \mathbb{E}[X^2_t]=\bar{T} \left(\mathbb{E}[{X}^2_1]\right)$. Now observe that 
\begin{equation}
    \begin{aligned}
    \mathbb{E}[{X}^2_1] &= \mathbb{E}\left[{\left(a_1a'_1- \bar{\Sigma}\right)}^2\right]\\
    &=\mathbb{E}\left[{\left(a_1a'_1\right)}^2\right]- {\bar{\Sigma}}^2 \\
    &\preccurlyeq  \mathbb{E}\left[{\left(a_1a'_1\right)}^2\right] \\
    &=\mathbb{E}\left[a_1 a'_1 a_1 a'_1\right] \\
    &= \mathbb{E}\left[\left({\Vert a_1 \Vert}^2\right) a_1 a'_1 \right] \\
    &= \frac{I_d}{d}, \\
    \end{aligned}
\end{equation}
where we used the fact that ${\Vert a_1 \Vert}^2=1$. We thus have
\begin{equation}
    Y \preccurlyeq\frac{\bar{T}}{d} I_d,
\end{equation}
and hence,
\begin{equation}
    {\norm[\bigg]  {\sum\limits_{t=1}^{\bar{T}} \mathbb{E}[X^2_t]} }_2 = \lambda_{\max}(Y) \leq \lambda_{\max}\left(\frac{\bar{T}}{d} I_d\right) = \frac{\bar{T}}{d}. 
\end{equation}
Setting $K=1$, $N=\bar{T}$, $\epsilon=\bar{T}/{2d}$, using $\sigma^2 \leq \bar{T}/d$, and appealing to Lemma \ref{lemma:mat_bern}, we obtain 
\begin{equation}
    \mathbb{P}\left( {\norm[\bigg] { \sum_{t=1}^{\bar{T}} X_t} }_2 \geq \frac{\bar{T}}{2d} \right) \leq \delta, \hspace{2mm} \textrm{if} \hspace{2mm} \bar{T} \geq \frac{28d}{3} \log\left(\frac{2d}{\delta}\right).
\end{equation}
Suppose $\bar{T}$ satisfies the condition above, and let $Z=\sum_{t=1}^{\bar{T}} X_t$. We then have that with probability at least $1-\delta$,
\begin{equation}
    {\Vert Z \Vert}_2 < \frac{\bar{T}}{2d}.
\label{eqn:Z_bnd}
\end{equation}
Since 
\begin{equation}
    \sum_{t=1}^{\bar{T}} a_t a'_t = Z + \bar{T} \bar{\Sigma},
\end{equation}
it holds that
\begin{equation}
    \psi\left(\sum_{t=1}^{\bar{T}} a_t a'_t \right) =\psi(Z) + \frac{\bar{T}}{d},
\label{eqn:eigen}
\end{equation}
where we used $\psi(A)$ to denote an arbitrary eigenvalue of a matrix $A$. Now notice that
$$ \vert \psi(Z) \vert \leq \rho(Z) =  {\Vert Z \Vert}_2 < \frac{\bar{T}}{2d}, $$ 
where we used the fact that $Z$ is symmetric, and \eqref{eqn:Z_bnd}. We also used $\rho(Z)$ to denote the spectral radius of $Z$. The above inequality along with \eqref{eqn:eigen} immediately leads to the following: 
$$ \lambda_{\min}\left(\sum\limits_{t=1}^{\bar{T}} a_t a'_t \right) \geq \frac{\bar{T}}{2d}. $$
Combining all the above pieces, we have that with probability at least $1-\delta$,
$$ \lambda_{\min}(V_t) \geq \lambda_{\min}(V_{\bar{T}}) \geq  \lambda_{\min}\left(\sum\limits_{t=1}^{\bar{T}} a_t a'_t \right) \geq \frac{\bar{T}}{2d}, \forall t \geq \bar{T},$$
which is the desired conclusion. To arrive at \eqref{eqn:min_eigen}, we simply note that the choice $\bar{T}=\ceil{10L^2d\sqrt{T}\log(dLT)}$ satisfies the criterion that $\bar{T} \geq \frac{28d}{3} \log(2d/\delta)$, with $\delta=2/T$.  
\end{proof}

The next simple lemma concerns the deviation of the noise sequence.

\begin{lemma}
\label{lemma:noise}
The following is true: 
\begin{equation}
    \mathbb{P}\left( \exists t \in [T]:  \vert \eta_t \vert \geq \sqrt{4\log(T)} \right) \leq \frac{2}{T}.
\end{equation}
\end{lemma}
\begin{proof}
Recalling that each $\eta_t$ is 1-subgaussian, and applying the union bound, we obtain
\begin{equation}
    \begin{aligned}
    \mathbb{P}\left( \exists t \in [T]:  \vert \eta_t \vert \geq \sqrt{4\log(T)} \right) &\leq \sum_{t\in[T]} \mathbb{P}\left( \vert \eta_t \vert \geq \sqrt{4\log(T)} \right) \\
    &\leq T \times \frac{2}{T^2},
    \end{aligned}
\end{equation}
where the last inequality follows from invoking the Hoeffding bound for subgaussian random variables; see for instance \cite[Theorem 5.3]{tor}.
\end{proof}

The next step is to construct a ``clean event" for our subsequent analysis. To this end, consider the following three events.

\begin{equation}
    \begin{aligned}
    \mathcal{G}_1 &=\{\lambda_{\min}(V_t) \geq 5 L^2 \sqrt{T} \log(dLT), \forall t\geq \bar{T}\}. \\
    \mathcal{G}_2 &= \{\vert \eta_t \vert < \sqrt{4\log(T)}, \forall t \in [T]\}. \\ 
    \mathcal{G}_3 &= \{\theta_* \in \mathcal{C}_t, \forall t \in [T] \},
    \end{aligned}
\label{eqn:events}
\end{equation}
where $\mathcal{C}_t$ is as in \eqref{eqn:enc_conf_region}. Let $\mathcal{G}=\mathcal{G}_1 \cap \mathcal{G}_2 \cap \mathcal{G}_3$. In words, $\mathcal{G}$ is an event on which (i) the minimum eigenvalue of the covariance matrix is suitably bounded from below, (ii) the noise sequence is well-behaved, and (iii) the true parameter $\theta_*$ lies in a known confidence region. Setting $\delta=1/T$ in Lemma \ref{lemma:enc_conf_region}, and using Lemma's \ref{lemma:eigen_cov} and \ref{lemma:noise}, we immediately obtain that
\begin{equation}
    \mathbb{P}(\mathcal{G}) \geq 1-\frac{5}{T}.
\label{eqn:prob_clean_event}
\end{equation}

In what follows, we will condition on the clean event $\mathcal{G}$. Our next step is to bound the gap between consecutive estimates of $\theta_*$ at the agent.

\begin{lemma} On the event $\mathcal{G}$, the following holds $\forall t \geq \bar{T}$:
\begin{equation}
    {\Vert \hat{\theta}^{(a)}_{t+1}-\hat{\theta}^{(a)}_{t} \Vert}_2 \leq f(T), \hspace{2mm} \textrm{where} \hspace{2mm} f(T) \triangleq \frac{3}{5L} \sqrt{\frac{\beta_T}{T\log(dLT)}}.
\label{eqn:iterate_bound}
\end{equation}
\label{lemma:iterate_bound}
\end{lemma}
\begin{proof}
We first develop a recursion relating $\hat{\theta}^{(a)}_{t+1}$ to $\hat{\theta}^{(a)}_{t}$. Based on \eqref{eqn:least_square}, observe that 
\begin{equation}
    \begin{aligned}
    \hat{\theta}^{(a)}_{t+1} &= V^{-1}_{t+1} \left(\sum\limits_{s=1}^{t+1} a_s y_s \right) \\
    &= V^{-1}_{t+1} \left(\sum\limits_{s=1}^{t} a_s y_s + a_{t+1} y_{t+1} \right) \\
    &= V^{-1}_{t+1} \left(V_t \hat{\theta}^{(a)}_{t} + a_{t+1} y_{t+1} \right) \\
    &= V^{-1}_{t+1} \left(\left(V_{t+1}-a_{t+1}a'_{t+1}\right) \hat{\theta}^{(a)}_{t} + a_{t+1} y_{t+1} \right) \\
    &= \hat{\theta}^{(a)}_{t} + V^{-1}_{t+1} a_{t+1} \left(y_{t+1}-\langle a_{t+1}, \hat{\theta}^{(a)}_{t} \rangle \right)\\
    &= \hat{\theta}^{(a)}_{t} + V^{-1}_{t+1} a_{t+1} \left(y_{t+1}-\langle a_{t+1}, \theta_* \rangle + \langle \theta_* -\hat{\theta}^{(a)}_{t}, a_{t+1} \rangle  \right) \\
    &= \hat{\theta}^{(a)}_{t} + V^{-1}_{t+1} a_{t+1} \eta_{t+1} + V^{-1}_{t+1} a_{t+1} \left(\langle \theta_* -\hat{\theta}^{(a)}_{t}, a_{t+1} \rangle \right).
    \end{aligned}
\label{eqn:iterate_recursion}
\end{equation}
For the last step, we used the expression of the observation model in \eqref{eqn:obs_model}. From \eqref{eqn:iterate_recursion}, we immediately obtain:
\begin{equation}
    {\Vert{\hat{\theta}^{(a)}_{t+1}-\hat{\theta}^{(a)}_{t}}\Vert}_2 \leq \underbrace{\norm[\bigg]{ V^{-1}_{t+1} a_{t+1} \eta_{t+1}}_2}_{T_1} + \underbrace{\norm[\bigg]{ V^{-1}_{t+1} a_{t+1} \left(\langle \theta_* -\hat{\theta}^{(a)}_{t}, a_{t+1} \rangle \right)}_2}_{T_2}.
\label{eqn:iterate_interim}
\end{equation}
We now proceed to bound each of the terms $T_1$ and $T_2$ separately. To this end, suppose $t\geq \bar{T},$ and note that we are on the event $\mathcal{G}$. We then have
\begin{equation}
    \begin{aligned}
    T_1 & \leq \vert \eta_{t+1} \vert {\Vert a_{t+1} \Vert}_2 {\Vert V^{-1}_{t+1} \Vert}_2\\
    & \leq L \sqrt{4\log(T)} {\Vert V^{-1}_{t+1} \Vert}_2,
    \end{aligned}
\label{eqn:T1_interim}
\end{equation}
where we used the property of event $\mathcal{G}_2$ in \eqref{eqn:events} and the bound on the actions from Assumption \ref{ass:actions}. Appealing to the property of event $\mathcal{G}_1$ in \eqref{eqn:events}, we have:
\begin{equation}
    {\Vert V^{-1}_{t+1} \Vert}_2 = \lambda_{\max}(V^{-1}_{t+1}) = \frac{1}{\lambda_{\min}(V_{t+1})} \leq \frac{1}{5 L^2 \sqrt{T} \log(dLT)}.
\label{eqn:bound_V_norm}
\end{equation}
Combining the above bound with that in \eqref{eqn:T1_interim}, we obtain
\begin{equation}
\begin{aligned}
    T_1 &\leq  \frac{2\sqrt{\log(T)}}{5L\sqrt{T} \log(dLT)} \\
    &\leq \frac{2\sqrt{\log(dLT)}}{5L\sqrt{T} \log(dLT)} \\
    & = \frac{2}{5L\sqrt{T\log(dLT)}},
\end{aligned}
\label{eqn:T_1}
\end{equation}
where for the second step, we used $dL \geq 1$. For bounding the term $T_2$, we proceed as follows:
\begin{equation}
    \begin{aligned}
    T_2 & \leq  {\Vert V^{-1}_{t+1} \Vert}_2 {\Vert a_{t+1} \Vert}_2 \left| \left(\langle \theta_* -\hat{\theta}^{(a)}_{t}, a_{t+1} \rangle \right) \right| \\
    & \overset{(a)}\leq \frac{1}{5 L \sqrt{T} \log(dLT)} \left| \left(\langle \theta_* -\hat{\theta}^{(a)}_{t}, a_{t+1} \rangle \right) \right| \\
    & \leq \frac{1}{5 L \sqrt{T} \log(dLT)} {\Vert \theta_* -\hat{\theta}^{(a)}_{t} \Vert}_{V_t} {\Vert a_{t+1} \Vert}_{V^{-1}_t} \\
    & \overset{(b)} \leq \frac{\sqrt{\beta_T}}{5 L \sqrt{T} \log(dLT)} {\Vert a_{t+1} \Vert}_{V^{-1}_t}.
    \end{aligned}
\label{eqn:T_2_interim}
\end{equation}
In the above steps, we used \eqref{eqn:bound_V_norm} and the bound on the actions for (a); and the property of the event $\mathcal{G}_3$ for (b). As for the term ${\Vert a_{t+1} \Vert}_{V^{-1}_t}$, observe
\begin{equation}
    \begin{aligned}
    {\Vert a_{t+1} \Vert}_{V^{-1}_t} &= \sqrt{a'_{t+1} V^{-1}_t a_{t+1}} \\
    &\leq \sqrt{ {\Vert a_{t+1} \Vert}^2_2 \,  \lambda_{\max}(V^{-1}_t) } \\
    &\leq \sqrt{\frac{1}{5\sqrt{T}\log(dLT)}}.
    \end{aligned}
\end{equation}
For the last step, we used \eqref{eqn:bound_V_norm}. Combining the above inequality with that in \eqref{eqn:T_2_interim}, we obtain
\begin{equation}
    T_2 \leq \frac{\sqrt{\beta_T}}{L{\left(5\sqrt{T}\log(dLT)\right)}^{\frac{3}{2}}}. 
\label{eqn:T_2}
\end{equation}
Putting together the bounds in \eqref{eqn:iterate_interim}, \eqref{eqn:T_1}, and \eqref{eqn:T_2}, and simplifying, we obtain the bound in \eqref{eqn:iterate_bound}. This concludes the proof. 
\end{proof}

The next lemma justifies the adaptive encoding strategy outlined in Algorithm \ref{algo:Encoder}.

\begin{lemma} (\textbf{Encoding Region}) On the event $\mathcal{G}$, the following is true:
$$ e_t \in \mathcal{B}_{d}(0,p_t), \forall t \in \{\bar{T}+1, \ldots, T\},$$
where $e_t$ is the innovation in line 2 of Algorithm \ref{algo:Encoder}, and $p_t$ is as defined in Eq.  \eqref{eqn:quantizer_eqs}.
\end{lemma}
\begin{proof}
We will prove this result by induction. We start by establishing the base case.

\textbf{Base Case.} Our goal is to establish that $e_{\bar{T}+1} \in \mathcal{B}_d(0,p_{\bar{T}+1})$. Based on the property of event $\mathcal{G}_3$, let us note the following:
\begin{equation}
    \begin{aligned}
    & {\Vert \theta_* - \hat{\theta}^{(a)}_{\bar{T}+1} \Vert}^2_{V_{\bar{T}+1}} \leq \beta_{T} \\ 
    & \implies {\left(\theta_* - \hat{\theta}^{(a)}_{\bar{T}+1}\right)}' V_{\bar{T}+1} \left(\theta_* - \hat{\theta}^{(a)}_{\bar{T}+1}\right) \leq \beta_T \\ 
    & \implies \lambda_{\min}(V_{\bar{T}+1}) {\Vert \theta_* - \hat{\theta}^{(a)}_{\bar{T}+1} \Vert}^2_{2} \leq \beta_T \\
    & \implies {\Vert \theta_* - \hat{\theta}^{(a)}_{\bar{T}+1} \Vert}_{2} \leq \sqrt{\frac{\beta_T}{\lambda_{\min}(V_{\bar{T}+1})}} \\
    & \implies {\Vert \hat{\theta}^{(a)}_{\bar{T}+1} \Vert}_2 \leq {\Vert \theta_* \Vert}_2 + \sqrt{\frac{\beta_T}{\lambda_{\min}(V_{\bar{T}+1})}}\\
    & \implies {\Vert \hat{\theta}^{(a)}_{\bar{T}+1} \Vert}_2 \leq M + \sqrt{\frac{\beta_T}{5 L^2 \sqrt{T} \log(dLT)}}, \\
    \end{aligned}
\label{eqn:initial_error}
\end{equation}
where for the last step, we used \eqref{eqn:bound_V_norm}, and the fact that ${\Vert \theta_* \Vert}_2 \leq M$ since $\theta_* \in \Theta$. To proceed, we simplify the resulting expression above by plugging in the expression for $\beta_T$ from Eq. \eqref{eqn:conf_radius} with $\delta=1/T$. This yields:
\begin{equation}
    \begin{aligned}
    \sqrt{\frac{\beta_T}{5 L^2 \sqrt{T} \log(dLT)}} &= \frac{\sqrt{\lambda}M}{\sqrt{{5 L^2 \sqrt{T} \log(dLT)}}} + \sqrt{\frac{2\log\left(T\right)+d\log\left(\frac{d\lambda+TL^2}{d\lambda}\right)}{5 L^2 \sqrt{T} \log(dLT)}} \\
    &\leq \sqrt{\lambda}M + \sqrt{\frac{2}{5L^2 \sqrt{T}} + \frac{4d}{5L^2 \sqrt{T}}} \\
    &\leq \left(\sqrt{\lambda} + \sqrt{\frac{6d}{5L^2 \sqrt{T}} }\right) M \\
    & \leq 3M.
    \end{aligned}
\label{eqn:clean_up}
\end{equation}
In the above steps, we set the regularization parameter $\lambda=1$, and used the following facts: $M \geq 1$, $L \geq 1$, $\log(dLT) \geq 1$, and $T \geq d^2$. Combining the above bound with that in \eqref{eqn:initial_error}, we obtain
$$ {\Vert \hat{\theta}^{(a)}_{\bar{T}+1} \Vert}_2 \leq 4M. $$ We thus have
\begin{equation}
\begin{aligned}
  {\Vert  e_{\bar{T}+1} \Vert}_2 &= \Vert{\hat{\theta}^{(a)}_{\bar{T}+1} - \hat{\theta}^{(s)}_{\bar{T}}\Vert}_2 \\
  & \leq \Vert{\hat{\theta}^{(a)}_{\bar{T}+1}  \Vert}_2 + \Vert{\hat{\theta}^{(s)}_{\bar{T}}\Vert}_2 \\
  &\leq 4M + M \\
  & \leq 5M + \frac{3}{2} f(T) = p_{\bar{T}+1}, 
 \end{aligned} 
\end{equation}
where we used the fact that $\hat{\theta}^{(s)}_{\bar{T}} \in \Theta$, and hence, $\Vert{\hat{\theta}^{(s)}_{\bar{T}}\Vert}_2 \leq M$. We have thus established that ${\Vert  e_{\bar{T}+1} \Vert}_2 \leq p_{\bar{T}+1}$, implying that $e_{\bar{T}+1} \in \mathcal{B}_d(0,p_{\bar{T}+1}).$ This completes the base case. 

\textbf{Induction Step.} Now suppose for all $t\in\{\bar{T}+1, \ldots, k\}$, it holds that $e_t \in \mathcal{B}_d(0,p_t)$. Here, $k\in \{\bar{T}+2, \ldots, T-1\}$. Our goal is to establish that $e_{k+1} \in \mathcal{B}_d(0,p_{k+1})$. Based on the induction hypothesis, $e_k \in \mathcal{B}_d(0,p_k)$, and hence, the encoding operation outlined in Algorithm \ref{algo:Encoder} is valid at time-step $k$. Recall that $\tilde{e}_k$ is the center of the ball (of radius $\epsilon p_k$) that contains $e_k$. Based on the definition of an $\epsilon p_k$ net, we then immediately have
\begin{equation}
    {\Vert \tilde{e}_k-e_k \Vert}_2 \leq \epsilon p_k = \frac{1}{2} p_{k},
\label{eqn:enc_err}
\end{equation}
with $\epsilon$ set to $1/2$. Now observe
\begin{equation}
    \begin{aligned}
    \hat{\theta}^{(s)}_k &=  \hat{\theta}^{(s)}_{k-1} + \tilde{e}_k \\
    &=\hat{\theta}^{(s)}_{k-1} + e_k + \tilde{e}_k - e_k \\
   &= \hat{\theta}^{(s)}_{k-1} + \left(\hat{\theta}^{(a)}_{k}-\hat{\theta}^{(s)}_{k-1}\right) + \tilde{e}_k - e_k \\
  &= \hat{\theta}^{(a)}_{k} + \tilde{e}_k - e_k.
    \end{aligned}
\end{equation}
Based on the above display and \eqref{eqn:enc_err}, we conclude 
\begin{equation}
 {\Vert \hat{\theta}^{(s)}_k - \hat{\theta}^{(a)}_k \Vert}_2 = {\Vert \tilde{e}_k-e_k \Vert}_2  \leq \frac{1}{2} p_k.
\label{eqn:dec_error}
\end{equation}
We then have
\begin{equation}
    \begin{aligned}
    e_{k+1} & = \hat{\theta}^{(a)}_{k+1}-\hat{\theta}^{(s)}_{k}\\
    &= \left(\hat{\theta}^{(a)}_{k+1}-\hat{\theta}^{(a)}_k\right) + \left(\hat{\theta}^{(a)}_k - \hat{\theta}^{(s)}_{k}\right).
    \end{aligned}
\end{equation}
The triangle inequality then yields
\begin{equation}
    \begin{aligned}
    {\Vert e_{k+1} \Vert}_2 &\leq {\Vert \hat{\theta}^{(a)}_{k+1}-\hat{\theta}^{(a)}_k \Vert}_2 + {\Vert \hat{\theta}^{(a)}_k - \hat{\theta}^{(s)}_{k} \Vert}_2 \\
    & \overset{(a)} \leq f(T) + \frac{1}{2} p_k \\
    & \overset{(b)} = f(T) + \frac{1}{2} \left(q_k + f(T)\right)\\
    & \overset{(c)} = f(T)+q_{k+1} \\
    & \overset{(d)} = p_{k+1}.
    \end{aligned}
\end{equation}
In the above steps, we used Lemma \ref{lemma:iterate_bound} and Eq. \eqref{eqn:dec_error} for (a), and the definitions of $p_k, q_k, p_{k+1}$, and $q_{k+1}$ from Eq. \eqref{eqn:quantizer_eqs} for (b)-(d). Since ${\Vert e_{k+1} \Vert}_2 \leq p_{k+1}$, it follows that $e_{k+1} \in \mathcal{B}_d(0,p_{k+1})$, thereby establishing the induction claim. This completes the proof. 
\end{proof}

The next key result justifies the decision-making rule in lines 8-9 of the \texttt{IC-LinUCB} algorithm. 

\begin{lemma} (\textbf{Confidence Region at Server}) On the event $\mathcal{G}$, the following is true: $\theta_* \in \mathcal{C}^{(s)}_t, \forall t \in \{\bar{T}+2, \ldots, T\}$, where
$$
    \mathcal{C}^{(s)}_t = \{\theta\in\mathbb{R}^d: {\Vert \theta - \hat{\theta}^{(s)}_{t-1} \Vert}_{V_{t-1}} \leq \sqrt{\beta_T} + \textcolor{black}{\left(\sqrt{\lambda+(t-1)L^2}\right) q_t\}}.
$$
Moreover, $\forall t \geq \bar{T}+\tilde{T}$, we have 
\begin{equation}
   \left(\sqrt{\lambda+(t-1)L^2}\right) q_t \leq 4 \sqrt{\frac{\beta_T}{\log(dLT)}},
\label{eqn:inflation_bound}
\end{equation}
where 
\begin{equation}
    \tilde{T} =  \ceil*{\frac{\log\left(\frac{10M}{f(T)}\right)}{\log(2)}} \vee 2 = O\left(\log(dLT)\right).
\label{eqn:tilde_T}
\end{equation}
\end{lemma}

\begin{proof}
Consider any $t\in\{\bar{T}+2, \ldots, T\}$. We then have
\begin{equation}
    \begin{aligned}
    {\Vert \hat{\theta}^{(s)}_{t-1} - \theta^* \Vert}_{V_{t-1}} & \overset{(a)} \leq {\Vert \hat{\theta}^{(a)}_{t-1} - \theta^* \Vert}_{V_{t-1}} + {\Vert \hat{\theta}^{(s)}_{t-1} - \hat{\theta}^{(a)}_{t-1} \Vert}_{V_{t-1}} \\
    & \overset{(b)}\leq \sqrt{\beta_T} + \sqrt{\left(\hat{\theta}^{(s)}_{t-1} - \hat{\theta}^{(a)}_{t-1}\right)' V_{t-1} \left(\hat{\theta}^{(s)}_{t-1} - \hat{\theta}^{(a)}_{t-1}\right)              } \\
    & \leq \sqrt{\beta_T} + \sqrt{\lambda_{\max}(V_{t-1}) {\Vert  \hat{\theta}^{(s)}_{t-1} - \hat{\theta}^{(a)}_{t-1} \Vert}^2_2             } \\
    & \overset{(c)} \leq \sqrt{\beta_T} + \sqrt{\lambda_{\max}(V_{t-1})} \left(\frac{1}{2}{p_{t-1}}\right) \\
    & \overset{(d)}=  \sqrt{\beta_T} + \sqrt{\lambda_{\max}(V_{t-1})} \left(\frac{1}{2}\left(q_{t-1}+f(T)\right)\right)\\
    & \overset{(e)}=  \sqrt{\beta_T} + \left(\sqrt{\lambda_{\max}(V_{t-1})}\right) q_t \\
    & \leq \sqrt{\beta_T} +  \left(\sqrt{\lambda+(t-1)L^2}\right)  q_t.
    \end{aligned}
\label{eqn:conf_interim}
\end{equation}
In the above steps, (a) follows from the triangle inequality; (b) follows from the defining property of event $\mathcal{G}_3$; (c) follows from Eq. \eqref{eqn:dec_error} in the analysis of Lemma \ref{lemma:enc_dec} (restated as Lemma 9 in the Appendix); and (d), (e) are a consequence of Eq.  \eqref{eqn:quantizer_eqs}. This establishes the first claim of the lemma.

For the second claim of the lemma, we roll out the recursion $q_{t}= \epsilon (q_{t-1}+f(T))$ for $\tau >0$ time-steps starting from $\bar{T}$:
\begin{equation}
\begin{aligned}
    q_{\bar{T}+\tau} &= {\epsilon}^{\tau} q_{\bar{T}} + \left(\sum\limits_{k=1}^{\tau} \epsilon^{k} \right) f(T) \\
    & \leq {\epsilon}^{\tau} q_{\bar{T}} + \left(\sum\limits_{k=0}^{\infty} \epsilon^{k} \right) f(T) \\
    & = {\epsilon}^{\tau} q_{\bar{T}} +  \frac{f(T)}{1-\epsilon} \\
    & = 10 M{\left(\frac{1}{2}\right)}^{\tau} + 2 f(T),
\end{aligned}    
\end{equation}
where for the last step, we used $\epsilon = 1/2$, and $q_{\bar{T}}=10M$. Now it is easy to verify that when 
$$ \tau \geq \ceil*{\frac{\log\left(\frac{10M}{f(T)}\right)}{\log(2)}}, $$
it holds that
$$ 10 M{\left(\frac{1}{2}\right)}^{\tau} \leq 2 f(T),$$
and hence, 
$q_{\bar{T}+\tau} \leq 4 f(T)$. Thus, for $t\geq \bar{T} + \tilde{T}$, where $\tilde{T}$ is as defined in Eq. \eqref{eqn:tilde_T}, we have 
\begin{equation}
\begin{aligned}
    \left(\sqrt{\lambda+(t-1)L^2}\right)  q_t & \leq 4 \left(\sqrt{\lambda+ TL^2}\right)  f(T) \\
    & = \frac{12}{5} \sqrt{ \frac{(\lambda + TL^2) \beta_T}{T L^2 \log(dLT)}} \\
    & \leq 4 \sqrt{\frac{\beta_T}{\log(dLT)}},
\end{aligned}
\end{equation}
where for the last step, we used $\lambda=1$, and $TL^2 \geq 1$. The fact that $\tilde{T}=O(\log(dLT))$ follows by plugging in the expression for $f(T)$ from \eqref{eqn:iterate_bound}, and then some simple algebra. 
\end{proof}
    
We are now equipped with all the pieces required for proving Theorem \ref{thm:ICLinUCB}. 

\begin{proof} (\textbf{Proof of Theorem \ref{thm:ICLinUCB}}) Let us condition on the clean event $\mathcal{G}$, and consider a time-step $t\geq {\bar{T}}+\tilde{T}+1$. Let $\tilde{\theta}_t$ be such that 
$$
    (a_t, \tilde{\theta}_t) = \argmax_{(a,\theta) \in \mathcal{A}_t \times \mathcal{C}^{(s)}_t} \langle \theta, a \rangle,
$$
where $\mathcal{C}^{(s)}_t$ is as defined in Eq. \eqref{eqn:dec_conf_int}. Next, let $a^*_t=\argmax_{a\in\mathcal{A}_t} \langle \theta_*, a \rangle$ be an optimal action at time-step $t$. We now proceed to bound the instantaneous regret $r_t = \langle \theta_*, a^*_t - a_t \rangle$ as follows:
\begin{equation}
    \begin{aligned}
    r_t & \overset{(a)} \leq \langle a_t, \tilde{\theta}_t - \theta_* \rangle \\
    & \leq {\Vert a_t \Vert}_{V^{-1}_{t-1}} \, {\Vert \theta_* - \tilde{\theta}_t \Vert}_{V_{t-1}} \\
    & \overset{(b)} \leq 2\left(1+4 \sqrt{\frac{1}{\log(dLT)}}\right) {\Vert a_t \Vert}_{V^{-1}_{t-1}} \sqrt{\beta_T} \\
    & \leq 10 {\Vert a_t \Vert}_{V^{-1}_{t-1}} {\sqrt{\beta_T}}. \\
    \end{aligned}
\end{equation}
In the above steps, (a) follows from the fact that $\theta_* \in \mathcal{C}^{(s)}_t, \forall t \geq \bar{T}+2$ on the event $\mathcal{G}$ as per Lemma \ref{lemma:dec_conf_region}. Hence, $\langle \theta_*, a^*_t \rangle \leq \langle \tilde{\theta}_t, a_t \rangle$. For (b), we invoked the bound from \eqref{eqn:inflation_bound}, and combined it with Eq. \eqref{eqn:conf_interim}. Using the fact that $r_t \leq 10 \leq 10 \sqrt{\beta_T}$ (based on Assumption \ref{ass:actions}(i)), and combining it with the above bound, we finally obtain
$$ r_t  \leq 10 \sqrt{\beta_T} \left(1 \wedge {\Vert a_t \Vert}_{V^{-1}_{t-1}} \right). $$
We now follow certain standard steps.
\begin{equation}
    \begin{aligned}
    \sum_{t=\bar{T}+\tilde{T}+1}^{T} \left(1 \wedge {\Vert a_t \Vert}^2_{V^{-1}_{t-1}} \right) &\leq 2 \hspace{-2.5mm} \sum_{t=\bar{T}+\tilde{T}+1}^{T}\hspace{-2.5mm} \log\left(1+ {\Vert a_t \Vert}^2_{V^{-1}_{t-1}}\right)\\
    &= 2 \hspace{-2.5mm} \sum_{t=\bar{T}+\tilde{T}+1}^{T} \hspace{-2.5mm} \log\left(\frac{\det(V_t)}{\det(V_{t-1})}\right)\\
    & = 2 \log\left(\frac{\det(V_{{T}})}{\det(V_{\bar{T}+\tilde{T}})}\right),
    \end{aligned}
\end{equation}
where for the first two steps, we used \cite[Lemma 19.4]{tor}. Based on Jensen's inequality, we then have
\begin{equation}
    \begin{aligned}
    \sum_{t=\bar{T}+\tilde{T}+1}^{T}\hspace{-2.5mm} r_t &\leq \sqrt{\left(T-(\bar{T}+\tilde{T}) \right) \hspace{-2.5mm}  \sum_{t=\bar{T}+\tilde{T}+1}^{T} \hspace{-2.5mm}  r^2_t} \\
    &\leq 10\sqrt{\left(T-(\bar{T}+\tilde{T})\right) \beta_T \hspace{-2.5mm}  \sum_{t=\bar{T}+\tilde{T}+1}^{T} \hspace{-2.5mm}  \left(1 \wedge {\Vert a_t \Vert}^2_{V^{-1}_{t-1}} \right)} \\
    &\leq 10 \sqrt{2\left(T-(\bar{T}+\tilde{T})\right) \beta_T  \log\left(\frac{\det(V_{{T}})}{\det(V_{\bar{T}+\tilde{T}})}\right)   }.
    \end{aligned}
\label{eqn:regret_later}
\end{equation}
Based on Assumption \ref{ass:actions}(i), we note that $r_t$ is trivially at most $1$ at each time-step $t \in [\bar{T}+\tilde{T}]$. Thus, we have:
\begin{equation}
\begin{aligned}
    \sum_{t=1}^{\bar{T}+\tilde{T}} r_t &\leq \bar{T}+\tilde{T} \\
    &\leq \ceil*{10 L^2d\sqrt{T}\log(dLT)} + \ceil*{\frac{\log\left(\frac{10M}{f(T)}\right)}{\log(2)}} \vee 2\\
    & = O(L^2 d \sqrt{T} \log(dLT)).
\end{aligned}
\end{equation}
Combining the above bound with that in Eq. \eqref{eqn:regret_later}, we obtain the following bound on the total regret conditioned on the clean event $\mathcal{G}$:
\begin{equation}
    \sum_{t=1}^{T} r_t \leq O(L^2 d \sqrt{T} \log(dLT)) + \underbrace{10  \sqrt{2\left(T-(\bar{T}+\tilde{T})\right) \beta_T  \log\left(\frac{\det(V_{{T}})}{\det(V_{\bar{T}+\tilde{T}})}\right)}}_{T_3}.
\end{equation}
For bounding $T_3$, we note based on the AM-GM inequality that
$$ \det(V_T) = \prod_{i=1}^{d}\lambda_i(V_T) \leq {\left(\frac{1}{d} \textrm{Trace}(V_T)\right)}^d \leq {\left(\frac{d \lambda+TL^2}{d}\right)}^d. $$
Moreover, note that $\det(V_{\bar{T}+\tilde{T}}) \geq \det(V_0)=\lambda^d$.\footnote{At this stage, we could have used a tighter lower bound on $\det(V_{\bar{T}+\tilde{T}})$ as follows: $\det(V_{\bar{T}+\tilde{T}}) \geq \det(V_{\bar{T}}) \geq {\left(\lambda_{\min}(V_{\bar{T}})\right)}^d \geq {\left(5 L^2 \sqrt{T} \log(dLT)\right)}^d$, where the last inequality follows from \eqref{eqn:bound_V_norm}. While this will improve the bound for $T_3$, the  improvement is of no consequence since $T_3$ will be dominated by the additive term $O(L^2 d \sqrt{T} \log(dLT))$.} Combining this bound with the one above for $\det(V_T)$, plugging in the expression for $\beta_T$ from Eq. \eqref{eqn:conf_radius} in $T_3$, and then simplifying, it is easy to verify that $T_3= O(d\sqrt{T}\log(LT))$. We have thus argued that on the clean event $\mathcal{G}$, the overall regret is of order $O(L^2 d \sqrt{T} \log(dLT)) + O(d\sqrt{T}\log(LT))=O(L^2 d \sqrt{T} \log(dLT))$.

Let $\mathcal{I}_{\mathcal{G}}$ represent an indicator random variable for the event $\mathcal{G}$. Moreover, recall from \eqref{eqn:prob_clean_event} that $\mathbb{P}(\mathcal{G}^{c}) \leq 5/T$. We now proceed to bound the regret $R_T$ as follows:
\begin{equation}
\begin{aligned}
    R_T &= \mathbb{E}\left[\sum_{t=1}^{T}r_t\right]\\
    &= \mathbb{E}\left[\left(\sum_{t=1}^{T}r_t\right)\left(\mathcal{I}_{\mathcal{G}}+\mathcal{I}_{\mathcal{G}^c}\right)\right]\\
    &= \mathbb{E}\left[\left(\sum_{t=1}^{T}r_t\right)\mathcal{I}_{\mathcal{G}}\right] + \mathbb{E}\left[\left(\sum_{t=1}^{T}r_t\right)\mathcal{I}_{\mathcal{G}^c}\right]\\
    & \overset{(a)}\leq CL^2d\sqrt{T}\log(dLT) + T \mathbb{P}(\mathcal{I}_{\mathcal{G}^c})\\
    & \overset{(b)}\leq CL^2d\sqrt{T}\log(dLT) + T \times \frac{5}{T}\\
    & = O(L^2d\sqrt{T}\log(dLT)),
\end{aligned}
    \end{equation}
where $C$ is a suitably large universal constant. To bound the second term in (a), we used the fact that $R_T$ is trivially bounded above by $T$. For (b), we used \eqref{eqn:prob_clean_event}. This establishes the claim regarding the regret. We now turn our attention to the capacity of the channel.

\textbf{Communication Complexity.} Recall that in the first $\bar{T}$ time-steps, there is no transmission from the agent to the server. For each $t \geq \bar{T}+1$, our encoding strategy involves constructing an $\epsilon p_t$ - net of the ball $\mathcal{B}_d(0,p_t)$. Invoking Lemma \ref{lemma:covers} with $\mathcal{K}=\mathcal{B}_d(0,p_t)$, we obtain:
\begin{equation}
    \begin{aligned}
     \mathcal{N}(\mathcal{B}_d(0,p_t),\epsilon\-p_t) &\leq \frac{\vert \left(\mathcal{B}_d(0,p_t)\oplus(\epsilon p_t/2)\mathcal{B}_d(0,1)\right)\vert}{\vert (\epsilon p_t/2)\mathcal{B}_d(0,1)\vert}\\
     &= \frac{\vert \left(1+\epsilon/2\right)p_t\mathcal{B}_d(0,1) \vert}{\vert (\epsilon p_t/2)\mathcal{B}_d(0,1)\vert}\\
     & = \frac{\left[{\left(1+\epsilon/2\right)p_t}\right]^d}{\left[\left(\epsilon/2\right)p_t\right]^d}\\
     & = \left(\frac{2}{\epsilon}+1\right)^d.
    \end{aligned}
\end{equation}
Thus, with $\epsilon=\frac{1}{2}$ in Algorithm \ref{algo:Encoder}, we obtain $\mathcal{N}(\mathcal{B}_d(0,p_t),\epsilon\-p_t) \leq 5^d$. To account for the possibility of an overflow, we need one additional symbol. Thus, the size of the alphabet $\Sigma$ needed to encode the innovation $e_t$ at any time-step is at most $5^d+1 < 5^{d+1}$. We conclude that it suffices for the channel capacity $B$ to satisfy 
$$B \geq \floor*{(d+1)\log_2(5)}+1. $$
Since $d \geq 1$, $B \geq 6d$ bits suffice. This completes the proof. 
\end{proof}
\newpage
\section{Analysis of the Generalized Linear Model under \\ Information Constraints}
\label{app:GLM_analysis}
In this section, we will analyze the information constrained \texttt{GLM-UCB} algorithm that was developed in Section \ref{sec:gen_lin}. Before doing so, we remind the reader of the observation model:
$$y_t=\mu\left(\langle \theta^*, a_t \rangle \right)+\eta_t,$$
and the invertible function $g_t(\theta)$ defined as follows:
$$ g_t(\theta)=\lambda \theta + \sum_{s=1}^{t}\mu(\langle \theta, a_s \rangle) a_s.$$

\begin{algorithm}[H]
\caption{Adaptive Encoding at the Agent for the Generalized Linear Model}
\label{algo:Encoder_GLM}  
 \begin{algorithmic}[1] 
\Statex \hspace{-5mm} \textbf{Input Parameters:} $\hat{\theta}^{(s)}_{\bar{T}}$ is any arbitrary vector in $\Theta$; $q_{\bar{T}}=2 \left(2+\frac{3}{\sqrt{k_1 k_2}}\right)M$; and $\bar{f}(T)=\frac{3}{5L} \sqrt{ \frac{\beta_T}{k_1 k_2 T\log(dLT)}}$.
\For {$t\in \{\bar{T}+1, \ldots, T\}$}
\State Observe $y_t$ and solve for  $\hat{\theta}^{(a)}_t$ satisfying $g_t(\hat{\theta}^{(a)}_t) = \sum_{s=1}^{t} y_s a_s$. Compute  $e_t=\hat{\theta}^{(a)}_t-\hat{\theta}^{(s)}_{t-1}$. 
\State Encode $e_t$ by constructing an $\epsilon p_t$-net of $\mathcal{B}_d(0,p_t)$, where
\begin{equation}
  q_{t}=\epsilon \left( q_{t-1} + \bar{f}(T)\right); \hspace{2mm}  p_t=q_t+\bar{f}(T).
\label{eqn:GLM_quantizer_eqs}
\end{equation}
\EndFor
\end{algorithmic}
 \end{algorithm}
 
\begin{algorithm}[H]
\caption{Information Constrained \texttt{GLMUCB}  (\texttt{IC-GLMUCB})}
\label{algo:IC-GLMUCB}  
 \begin{algorithmic}[1] 
\Statex \hspace{-5mm} \textbf{Input Parameters:} $\bar{T}=\ceil*{10 (k_2/k_1)  L^2d\sqrt{T}\log(dLT)}$.
\Statex \hspace{-5mm} \textbf{Phase I:} \textit{Pure Exploration}
\For {$t\in \{1, \ldots, \bar{T}+1\}$} 
\State Server plays $a_t \sim  \textrm{Unif}(\mathbb{S}^{d-1})$. 
\State Agent receives reward $y_t$ as per \eqref{eqn:gen_obs_model} and solves $\hat{\theta}^{(a)}_t$ satisfying $g_t(\hat{\theta}^{(a)}_t) = \sum_{s=1}^{t} y_s a_s.$
\EndFor
\State Agent encodes $e_{\bar{T}+1}=\hat{\theta}^{(a)}_{\bar{T}+1}-\hat{\theta}^{(s)}_{\bar{T}}$ as per Algo. \ref{algo:Encoder_GLM}, and transmits $\sigma_{\bar{T}+1}=\mathcal{E}_{\bar{T}+1}\left(e_{\bar{T}+1}\right)$. 
\Statex \hrulefill
\Statex \hspace{-5mm} \textbf{Phase II:} \textit{Information-Constrained Exploration-Exploitation}
 \For {$t\in \{\bar{T}+2, \ldots, T\}$}
\State Server decodes $\tilde{e}_{t-1} =\mathcal{D}_{t-1}(\sigma_{t-1})$, and generates $\hat{\theta}^{(s)}_{t-1}=\hat{\theta}^{(s)}_{t-2}+\tilde{e}_{t-1}$. 
\State Server constructs the following confidence set: 
$$
     \mathcal{\bar{C}}^{(s)}_t =  \{\theta\in \Theta: \mathcal{H}_{t-1}(\theta) \leq \sqrt{\beta_T} + k_2\textcolor{black}{\left(\sqrt{\lambda+(t-1)L^2}\right) q_t\}}, \hspace{1mm} \textrm{where} \hspace{1mm}  \mathcal{H}_t(\theta) \triangleq {\Vert g_t(\theta)-g_t(\hat{\theta}^{(s)}_t) \Vert}_{V^{-1}_t}, 
$$
where $q_t$, $V_{t-1}$, and $\sqrt{{\beta}_T}$ are given by \eqref{eqn:GLM_quantizer_eqs},  \eqref{eqn:least_square}, and \eqref{eqn:conf_radius}, respectively. 
\State Server plays action $a_t = \argmax_{a\in\mathcal{A}_t} \max_{\theta\in\mathcal{\bar{C}}^{(s)}_t} \langle \theta, a \rangle$. 
\State Agent receives reward $y_t$ as per \eqref{eqn:gen_obs_model} and solves $\hat{\theta}^{(a)}_t$ satisfying $g_t(\hat{\theta}^{(a)}_t) = \sum_{s=1}^{t} y_s a_s$.
\State Agent encodes the innovation $e_t=\hat{\theta}^{(a)}_t-\hat{\theta}^{(s)}_{t-1}$ as per Algo. \ref{algo:Encoder_GLM}, and transmits $\sigma_t = \mathcal{E}_t(e_t)$. 
\EndFor
\end{algorithmic}
 \end{algorithm}

\subsection{Proof of Theorem \ref{thm:ICGLMUCB}}
We start with the following fact that will be used repeatedly in our subsequent analysis. 
\begin{lemma}
\label{lemma:MVT}
Consider any $\theta_1, \theta_2 \in \Theta$, and any $t\in [T]$. There exists a symmetric positive definite matrix $G_t(\theta_1;\theta_2)$ satisfying $k_1 V_t \preccurlyeq G_t(\theta_1;\theta_2)  \preccurlyeq k_2 V_t$, such that:
\begin{equation}
    g_t(\theta_1)-g_t(\theta_2)=G_t(\theta_1;\theta_2) (\theta_1-\theta_2). 
\end{equation}
\end{lemma}
\begin{proof}
For any $\theta\in\Theta$, let us denote by $\nabla g_t(\theta)$ the Jacobian matrix of $g_t(\cdot)$ at $\theta$. Such a matrix exists based on Assumption \ref{ass:link}. Now based on the mean value theorem, $\exists \alpha \in (0,1)$ such that
$$ g_t(\theta_1)-g_t(\theta_2)=\left(\nabla g_t\left(\alpha \theta_1 + (1-\alpha) \theta_2\right)\right)  (\theta_1-\theta_2). $$
Let $\bar{\theta}=\alpha \theta_1 + (1-\alpha) \theta_2$, and $G_t(\theta_1;\theta_2)=\nabla g_t(\bar{\theta})$. It remains to argue that the matrix $G_t(\theta_1;\theta_2)$ so defined is symmetric and positive definite. To this end, let us note:
\begin{equation}
    \begin{aligned}
    \nabla g_t(\bar{\theta}) & \overset{(a)}= \lambda I + \sum_{s=1}^{t}\dot{\mu}(\langle \bar{\theta}, a_s \rangle) a_s a'_s \\
   & \overset{(b)} \succcurlyeq  k_1\left(\lambda I + \sum_{s=1}^{t} a_s a'_s \right) \\
    &= k_1 V_t.
    \end{aligned}
\end{equation}
Here, we used the definition of $g_t(\cdot)$ in \eqref{eqn:g_func} for (a), and for (b), we used Assumption \ref{ass:link} and the fact that $k_1 \leq 1$. The above steps reveal that $G_t(\theta_1;\theta_2)$ is symmetric and positive definite (as $V_t \succ 0$). The claim that $G_t(\theta_1;\theta_2) \preccurlyeq k_2 V_t$ follows a similar reasoning and relies on the fact that $\mu(\cdot)$ is $k_2$-Lipschitz.  
\end{proof}

The next result will be useful in the construction of the confidence region at the server.

\begin{lemma}
\label{lemma:GLM_conf_1}
Given any $\delta\in(0,1)$, the following holds with probability at least $1-\delta$:
$$ {\norm[\bigg]{g_t(\theta_*)- \sum_{s=1}^{t} y_s a_s}}_{V^{-1}_t} \leq \sqrt{\beta_T}, \forall t \in [T],$$ 
where $\beta_T$ is as defined in Eq. \eqref{eqn:conf_radius}. 
\end{lemma}

\begin{proof}
We have
\begin{equation}
    \begin{aligned}
        {\norm[\bigg]{g_t(\theta_*)- \sum_{s=1}^{t} y_s a_s}}_{V^{-1}_t} &\overset{(a)}= {\norm[\bigg]{\lambda \theta_* + \sum_{s=1}^{t}\mu(\langle \theta_*, a_s \rangle) a_s - \sum_{s=1}^{t} y_s a_s}}_{V^{-1}_t}\\
        &={\norm[\bigg]{\lambda \theta_* + \sum_{s=1}^{t}\left(\mu(\langle \theta_*, a_s \rangle)  - y_s\right) a_s}}_{V^{-1}_t}\\
        &\overset{(b)}={\norm[\bigg]{\lambda \theta_* - \sum_{s=1}^{t} \eta_s a_s}}_{V^{-1}_t}\\
        & \leq {\norm[\bigg]{\lambda \theta_*}}_{V^{-1}_t} + {\norm[\bigg]{\sum_{s=1}^{t} \eta_s a_s}}_{V^{-1}_t},\\
    \end{aligned}
\label{eqn:vector_valued_MG}
\end{equation}
where for (a), we used Eq. \eqref{eqn:g_func}, and for (b), we used Eq. \eqref{eqn:gen_obs_model}. Using ${\Vert \theta_* \Vert}_2 \leq M$ and $\lambda_{\min}(V_t) \geq \lambda$, it is easy to see that
$$ {\norm[\bigg]{\lambda \theta_*}}_{V^{-1}_t} \leq \sqrt{\lambda} M. $$
To bound the second term in the RHS of the resulting inequality in \eqref{eqn:vector_valued_MG}, we invoke Theorem 20.4 in \cite{tor}. This yields that with probability at least $1-\delta$, the following is true  $\forall t\in [T]$:
$$  {\norm[\bigg]{\sum_{s=1}^{t} \eta_s a_s}}_{V^{-1}_t} \leq  \sqrt{2\log\left(\frac{1}{\delta}\right)+\log\left(\frac{\det(V_T)}{\lambda^d}\right)}. $$
To complete the proof, we use the following fact derived in the analysis of Theorem \ref{thm:ICLinUCB}:
$$\det(V_T) \leq {\left(\frac{d \lambda+TL^2}{d}\right)}^d.$$
Putting all the above pieces together leads to the desired conclusion. 
\end{proof}

Since $\bar{T}=\ceil*{10 (k_2/k_1)  L^2d\sqrt{T}\log(dLT)}$ in the \texttt{IC-GLMUCB} algorithm, following the same reasoning as in Lemma \ref{lemma:eigen_cov}, we have that with probability at least $1-2/T$, 
\begin{equation}
    \lambda_{\min}(V_t) \geq 5 \frac{k_2}{k_1} L^2 \sqrt{T} \log(dLT), \forall t\geq \bar{T}.
\label{eqn:GLM_min_eigen}
\end{equation}

As in the analysis of Theorem \ref{thm:ICLinUCB}, we will now work on a clean event that is the intersection of the following three events.
\begin{equation}
    \begin{aligned}
    \mathcal{F}_1 &=\{\lambda_{\min}(V_t) \geq 5  (k_2/k_1) L^2 \sqrt{T} \log(dLT), \forall t\geq \bar{T}\}. \\
    \mathcal{F}_2 &= \{\vert \eta_t \vert < \sqrt{4\log(T)}, \forall t \in [T]\}. \\ 
    \mathcal{F}_3 &= \{ D_t(\theta_*) \leq \sqrt{\beta_T}, \forall t \in [T]\},
    \end{aligned}
\label{eqn:GLM_events}
\end{equation}
where 
$$ D_t(\theta)= {\norm[\bigg]{g_t(\theta)- \sum_{s=1}^{t} y_s a_s}}_{V^{-1}_t}.$$ Let $\mathcal{F}=\mathcal{F}_1 \cap \mathcal{F}_2 \cap \mathcal{F}_3$. Setting $\delta=1/T$ in Lemma \ref{lemma:GLM_conf_1}, and using Lemma's \ref{lemma:eigen_cov} and \ref{lemma:noise}, we immediately obtain that
\begin{equation}
    \mathbb{P}(\mathcal{F}) \geq 1-\frac{5}{T}.
\label{eqn:GLM_prob_clean_event}
\end{equation}

We now establish an analog of Lemma \ref{lemma:iterate_bound}. 

\begin{lemma} On the event $\mathcal{F}$, the following holds $\forall t \geq \bar{T}$:
\begin{equation}
    {\Vert \hat{\theta}^{(a)}_{t+1}-\hat{\theta}^{(a)}_{t} \Vert}_2 \leq \bar{f}(T), \hspace{2mm} \textrm{where} \hspace{2mm} \bar{f}(T) \triangleq \frac{3}{5L} \sqrt{ \frac{\beta_T}{k_1 k_2 T\log(dLT)}}.
\label{eqn:GLM_iterate_bound}
\end{equation}
\label{lemma:GLM_iterate_bound}
\end{lemma}
\begin{proof} We start by noting that based on Eq. \eqref{eqn:non_lin_LS}, $\hat{\theta}^{(a)}_{t+1}$ satisfies the following equation:
\begin{equation}
\begin{aligned}
    g_{t+1}(\hat{\theta}^{(a)}_{t+1}) &= \sum_{s=1}^{t+1} y_s a_s \\
    &= \sum_{s=1}^{t}y_s a_s + y_{t+1}a_{t+1}\\
    &= g_t(\hat{\theta}^{(a)}_{t})+y_{t+1}a_{t+1}.
\end{aligned}   
\label{eqn:EST_1}
\end{equation}
At the same time, in view of Eq. \eqref{eqn:g_func}, we have
\begin{equation}
\begin{aligned}
     g_{t+1}(\hat{\theta}^{(a)}_{t+1}) &= \lambda \hat{\theta}^{(a)}_{t+1} + \sum_{s=1}^{t+1} \mu(\langle \hat{\theta}^{(a)}_{t+1}, a_s \rangle ) a_s \\
     &= g_{t}(\hat{\theta}^{(a)}_{t+1}) + \mu(\langle \hat{\theta}^{(a)}_{t+1}, a_{t+1} \rangle ) a_{t+1}.
\end{aligned}
\end{equation}
Comparing the above equation with that in Eq. \eqref{eqn:EST_1}, we conclude:
\begin{equation}
    \begin{aligned}
    g_t(\hat{\theta}^{(a)}_{t+1}) - g_t(\hat{\theta}^{(a)}_{t}) &= \left(y_{t+1}- \mu(\langle \hat{\theta}^{(a)}_{t+1}, a_{t+1} \rangle )\right) a_{t+1} \\
    &= \left(y_{t+1}- \mu(\langle \theta_*, a_{t+1} \rangle )\right) a_{t+1} + \left(\mu(\langle \theta_*, a_{t+1} \rangle ) - \mu(\langle \hat{\theta}^{(a)}_{t+1}, a_{t+1} \rangle )\right) a_{t+1} \\
    &= \eta_{t+1} a_{t+1} + \left(\mu(\langle \theta_*, a_{t+1} \rangle ) - \mu(\langle \hat{\theta}^{(a)}_{t+1}, a_{t+1} \rangle )\right) a_{t+1},
    \end{aligned}
\label{eqn:EST_2}
\end{equation}
where for the last step, we used the observation model \eqref{eqn:gen_obs_model}. Now based on Lemma \ref{lemma:MVT}, we know that 
$$ g_t(\hat{\theta}^{(a)}_{t+1}) - g_t(\hat{\theta}^{(a)}_{t}) = G_t \left(\hat{\theta}^{(a)}_{t+1}-\hat{\theta}^{(a)}_{t}\right), $$
where $G_t$ is a symmetric positive definite matrix satisfying $k_1 V_t \preccurlyeq G_t \preccurlyeq k_2 V_t$.\footnote{Here, we have suppressed the dependence of $G_t$ on $\hat{\theta}^{(a)}_{t+1}$ and $\hat{\theta}^{(a)}_{t}$ since this is apparent from context. We will continue to do so to prevent cluttering the exposition.} From the above equation and Eq. \eqref{eqn:EST_2}, we then have
$$ \hat{\theta}^{(a)}_{t+1}-\hat{\theta}^{(a)}_{t} = G^{-1}_{t} \left( \eta_{t+1} a_{t+1} + \left(\mu(\langle \theta_*, a_{t+1} \rangle ) - \mu(\langle \hat{\theta}^{(a)}_{t+1}, a_{t+1} \rangle )\right) a_{t+1} \right). $$

Applying the triangle inequality to the above display, we obtain
$$ {\Vert  \hat{\theta}^{(a)}_{t+1}-\hat{\theta}^{(a)}_{t} \Vert}_2 \leq \underbrace{{\Vert G^{-1}_{t} \eta_{t+1} a_{t+1} \Vert}_2}_{T_1} + \underbrace{\norm[\bigg]{ G^{-1}_{t} \left(\mu(\langle \theta_*, a_{t+1} \rangle ) - \mu(\langle \hat{\theta}^{(a)}_{t+1}, a_{t+1} \rangle )\right) a_{t+1} }_2}_{T_2}.$$

We now proceed to bound each of the terms $T_1$ and $T_2$ separately. For bounding $T_1$, we note that as $G_t \succcurlyeq k_1 V_t$, it holds that $G^{-1}_t \preccurlyeq (1/k_1) V^{-1}_t$. Since $G^{-1}_t$ is symmetric and positive definite, we then have:
\begin{equation}
    \Vert{G^{-1}_t \Vert}_2 = \lambda_{\max}(G^{-1}_t) \leq \frac{1}{k_1} \lambda_{\max}(V^{-1}_t) = \frac{1}{k_1 \lambda_{\min}(V_t)} \leq \frac{1}{k_1 \lambda_{\min}(V_{\bar{T}})}.
\label{eqn:G_norm_bound}
\end{equation}

We can now bound $T_1$ as follows.
\begin{equation}
    \begin{aligned}
    T_1 & \leq \vert \eta_{t+1} \vert {\Vert a_{t+1} \Vert}_2 {\Vert G^{-1}_{t} \Vert}_2\\
    & \overset{(a)}\leq \frac{L \sqrt{4\log(T)}}{k_1 \lambda_{\min}(V_{\bar{T}})} \\
    & \overset{(b)}\leq \frac{2 \sqrt{\log(T)}}{5Lk_2 \sqrt{T} \log(dLT)} \\
    & \leq \frac{2}{5Lk_2 \sqrt{T \log(dLT)}}.
    \end{aligned}
\label{eqn:GLM_T1}
\end{equation}
In the above steps, we used the properties of event $\mathcal{F}$ and Eq. \eqref{eqn:G_norm_bound} for (a), and Eq. \eqref{eqn:GLM_min_eigen} for (b). Bounding the term $T_2$ requires a bit more work. Starting from the defining property of event $\mathcal{F}_3$, consider the following set of implications: 
\begin{equation}
\begin{aligned}
& {\norm[\bigg]{g_{t+1}(\theta_*)- \sum_{s=1}^{t+1} y_s a_s}}^2_{V^{-1}_{t+1}} \leq \beta_T\\
& \hspace{-10mm} \overset{(a)} \implies {\norm[\bigg]{g_{t+1}(\theta_*)- g_{t+1}(\hat{\theta}^{(a)}_{t+1})}}^2_{V^{-1}_{t+1}} \leq \beta_T\\
& \hspace{-10mm} \overset{(b)} \implies {\norm[\bigg]{G_{t+1}\left(\theta_*-\hat{\theta}^{(a)}_{t+1}\right)}}^2_{V^{-1}_{t+1}} \leq \beta_T \\
& \hspace{-10mm} \implies \left(\theta_*-\hat{\theta}^{(a)}_{t+1}\right)' G_{t+1} V^{-1}_{t+1} G_{t+1} \left(\theta_*-\hat{\theta}^{(a)}_{t+1}\right) \leq \beta_T \\ 
& \hspace{-10mm} \overset{(c)} \implies k_1 \left(\theta_*-\hat{\theta}^{(a)}_{t+1}\right)' G_{t+1} \left(\theta_*-\hat{\theta}^{(a)}_{t+1}\right) \leq \beta_T \\
& \hspace{-10mm} \overset{(d)} \implies {(k_1)}^2 \left(\theta_*-\hat{\theta}^{(a)}_{t+1}\right)' V_{t+1} \left(\theta_*-\hat{\theta}^{(a)}_{t+1}\right) \leq \beta_T \\
& \hspace{-10mm} \implies {(k_1)}^2 \lambda_{\min}(V_{t+1}) \norm[\bigg]{ \theta_*-\hat{\theta}^{(a)}_{t+1} }^2_2\leq \beta_T.\\
& \hspace{-10mm} \implies \Vert { \theta_* - \hat{\theta}^{(a)}_{t+1} \Vert}_2 \leq \frac{1}{k_1} \sqrt{\frac{\beta_T} { \lambda_{\min}(V_{t+1})}}.
\end{aligned}
\label{eqn:GLM_T2_interim}
\end{equation}
In the above steps, (a) follows from the definition of $\hat{\theta}^{(a)}_{t+1}$ in Eq. \eqref{eqn:non_lin_LS}; (b) follows from invoking Lemma \ref{lemma:MVT}; and (c), (d) both follow as a consequence of the fact that $G_{t+1} \succcurlyeq k_1 V_{t+1}$. 

We can now bound $T_2$ as follows.
\begin{equation}
    \begin{aligned}
    T_2 &= \norm[\bigg]{ G^{-1}_{t} \left(\mu(\langle \theta_*, a_{t+1} \rangle ) - \mu(\langle \hat{\theta}^{(a)}_{t+1}, a_{t+1} \rangle )\right) a_{t+1} }_2\\
    & \leq \left| \mu(\langle \theta_*, a_{t+1} \rangle ) - \mu(\langle \hat{\theta}^{(a)}_{t+1}, a_{t+1} \rangle ) \right| {\Vert G^{-1}_{t} \Vert}_2 {\Vert a_{t+1} \Vert}_2  \\
    & \overset{(a)} \leq k_2 \left| \langle \theta_*- \hat{\theta}^{(a)}_{t+1}, a_{t+1} \rangle \right| {\Vert G^{-1}_{t} \Vert}_2 {\Vert a_{t+1} \Vert}_2 \\
    & \leq k_2 \norm[\bigg]{\theta_*-\hat{\theta}^{(a)}_{t+1}}_2 {\Vert G^{-1}_{t} \Vert}_2 {\Vert a_{t+1} \Vert}^2_2 \\
    & \overset{(b)}\leq \frac{k_2 L^2}{{(k_1)}^2} \frac{\sqrt{\beta_T}}{ {\left( \lambda_{\min}(V_{\bar{T}})\right)}^{3/2} } \\
    & \overset{(c)} \leq \frac{k_2 L^2}{{(k_1)}^2} \frac{\sqrt{\beta_T}}{ {\left( 5 (k_2/k_1) L^2 \sqrt{T} \log(dLT)\right)}^{3/2} }.
    \end{aligned}
\label{eqn:GLM_T2}
\end{equation}
In the above steps, (a) follows from the fact that $\mu(\cdot)$ is $k_2$-Lipschitz; (b) follows from equations \eqref{eqn:G_norm_bound}, \eqref{eqn:GLM_T2_interim}, and the bound on the actions; and (c) follows from Eq. \eqref{eqn:GLM_min_eigen}. Combining the bounds on $T_1$ and $T_2$ from equations  \eqref{eqn:GLM_T1} and \eqref{eqn:GLM_T2} respectively, and simplifying, we immediately obtain the claim of the lemma. This concludes the proof. 
\end{proof}

Equipped with the above result, our next goal is to develop analogs of Lemma's \ref{lemma:enc_dec} and \ref{lemma:dec_conf_region}. 

\begin{lemma}  On the event $\mathcal{F}$, the following is true for the \texttt{IC-GLMUCB} algorithm:
$$ e_t \in \mathcal{B}_{d}(0,p_t), \forall t \in \{\bar{T}+1, \ldots, T\},$$
where $e_t$ is the innovation in line 2 of Algorithm \ref{algo:Encoder_GLM}, and $p_t$ is as defined in Eq.  \eqref{eqn:GLM_quantizer_eqs}.
\label{lemma:GLM_enc_dec}
\end{lemma}
\begin{proof}
The proof relies on the same induction technique employed in the analysis of Lemma \ref{lemma:enc_dec}. We only establish the base case since the proof of the induction step is identical to that in  Lemma \ref{lemma:enc_dec}. To establish the base case, we need to argue that $e_{\bar{T}+1} \in \mathcal{B}_d(0,p_{\bar{T}+1})$. Based on the arguments used to arrive at \eqref{eqn:GLM_T2_interim}, we obtain:
\begin{equation}
    \begin{aligned}
    \Vert { \theta_* - \hat{\theta}^{(a)}_{\bar{T}+1} \Vert}_2 & \leq \frac{1}{k_1} \sqrt{\frac{\beta_T} { \lambda_{\min}(V_{\bar{T}})}} \\
    & \overset{(a)} \leq \frac{1}{\sqrt{k_1 k_2}} \sqrt{ \frac{\beta_T} {5L^2 \sqrt{T} \log(dLT)}} \\
    & \overset{(b)} \leq \frac{3M}{\sqrt{k_1 k_2}},\\
    \end{aligned}
\end{equation}
where (a) follows from Eq. \eqref{eqn:GLM_min_eigen}, and (b) follows from the reasoning used to arrive at Eq. \eqref{eqn:clean_up}. This immediately implies that
$$ \norm[\bigg]{\hat{\theta}^{(a)}_{\bar{T}+1}}_2 \leq {\Vert \theta_* \Vert}_2 + \frac{3M}{\sqrt{k_1 k_2}} \leq \left(1+\frac{3}{\sqrt{k_1 k_2}}\right)M. $$ Finally, we have
\begin{equation}
\begin{aligned}
  {\Vert  e_{\bar{T}+1} \Vert}_2 &= \Vert{\hat{\theta}^{(a)}_{\bar{T}+1} - \hat{\theta}^{(s)}_{\bar{T}}\Vert}_2 \\
  & \leq \Vert{\hat{\theta}^{(a)}_{\bar{T}+1}  \Vert}_2 + \Vert{\hat{\theta}^{(s)}_{\bar{T}}\Vert}_2 \\
  &\leq \left(1+\frac{3}{\sqrt{k_1 k_2}}\right)M + M \\
  & \leq \left(2+\frac{3}{\sqrt{k_1 k_2}}\right)M + \frac{3}{2} f(T) = p_{\bar{T}+1}, 
 \end{aligned} 
\end{equation}
where we used the fact that $\hat{\theta}^{(s)}_{\bar{T}} \in \Theta$, and hence, $\Vert{\hat{\theta}^{(s)}_{\bar{T}}\Vert}_2 \leq M$. This establishes the desired claim and completes the proof. 
\end{proof}

The next result justifies the decision making rule of the \texttt{IC-GLMUCB} algorithm.

\begin{lemma}  On the event $\mathcal{F}$, the following is true: $\theta_* \in \bar{\mathcal{C}}^{(s)}_t, \forall t \in \{\bar{T}+2, \ldots, T\}$, where
$$
    \mathcal{\bar{C}}^{(s)}_t =  \{\theta\in \Theta: \mathcal{H}_{t-1}(\theta) \leq \sqrt{\beta_T} + k_2\textcolor{black}{\left(\sqrt{\lambda+(t-1)L^2}\right) q_t\}}; \hspace{1mm} \hspace{1mm}  \mathcal{H}_t(\theta) \triangleq {\Vert g_t(\theta)-g_t(\hat{\theta}^{(s)}_t) \Vert}_{V^{-1}_t}.
$$
Moreover, $\forall t \geq \bar{T}+\tilde{T}$, we have 
\begin{equation}
   \left(\sqrt{\lambda+(t-1)L^2}\right) q_t \leq 4 \sqrt{\frac{\beta_T}{k_1 k_2 \log(dLT)}},
\label{eqn:GLM_inflation_bound}
\end{equation}
where 
\begin{equation}
    \tilde{T} =  \ceil*{\frac{\log\left(\frac{(2+3/\sqrt{k_1 k_2})M}{\bar{f}(T)}\right)}{\log(2)}} \vee 2 = O\left(\log(dLT)\right).
\label{eqn:GLM_tilde_T}
\end{equation}
\label{lemma:GLM_dec_conf_region}
\end{lemma}

\begin{proof}
Consider any time-step $t\geq \bar{T}+2$, and observe:
\begin{equation}
    \begin{aligned}
    \norm[\bigg]{ g_{t-1}(\theta_*)-g_{t-1}(\hat{\theta}^{(s)}_{t-1}) }_{V^{-1}_{t-1}} & \overset{(a)} \leq \norm[\bigg]{\ g_{t-1}(\theta_*)- \sum_{s=1}^{t-1} y_s a_s  }_{V^{-1}_{t-1}} + \norm[\bigg]{ \sum_{s=1}^{t-1} y_s a_s-g_{t-1}(\hat{\theta}^{(s)}_{t-1}) }_{V^{-1}_{t-1}}\\
    & \overset{(b)} = \norm[\bigg]{\ g_{t-1}(\theta_*)- \sum_{s=1}^{t-1} y_s a_s  }_{V^{-1}_{t-1}} + \norm[\bigg]{ g_{t-1}(\hat{\theta}^{(a)}_{t-1})-g_{t-1}(\hat{\theta}^{(s)}_{t-1}) }_{V^{-1}_{t-1}}\\
    & \overset{(c)} \leq \sqrt{\beta_T} + \norm[\bigg]{ g_{t-1}(\hat{\theta}^{(a)}_{t-1})-g_{t-1}(\hat{\theta}^{(s)}_{t-1}) }_{V^{-1}_{t-1}}\\
    & \overset{(d)} = \sqrt{\beta_T} + \norm[\bigg]{G_{t-1} \left(\hat{\theta}^{(a)}_{t-1} - \hat{\theta}^{(s)}_{t-1}\right)  }_{V^{-1}_{t-1}} \\
    & = \sqrt{\beta_T} + \sqrt{ \left(\hat{\theta}^{(a)}_{t-1} - \hat{\theta}^{(s)}_{t-1}\right)' G_{t-1} V^{-1}_{t-1} G_{t-1} \left(\hat{\theta}^{(a)}_{t-1} - \hat{\theta}^{(s)}_{t-1}\right)}\\
    & \overset{(e)} \leq \sqrt{\beta_T} + \sqrt{ k_2  \left(\hat{\theta}^{(a)}_{t-1} - \hat{\theta}^{(s)}_{t-1}\right)' G_{t-1}  \left(\hat{\theta}^{(a)}_{t-1} - \hat{\theta}^{(s)}_{t-1}\right)} \\
    & \overset{(f)} \leq \sqrt{\beta_T} + k_2 \sqrt{  \left(\hat{\theta}^{(a)}_{t-1} - \hat{\theta}^{(s)}_{t-1}\right)' V_{t-1}  \left(\hat{\theta}^{(a)}_{t-1} - \hat{\theta}^{(s)}_{t-1}\right)} \\
    & \leq \sqrt{\beta_T} + k_2 \left(\sqrt{\lambda_{\max}(V_{t-1})}\right)  \norm[\bigg]{     \hat{\theta}^{(a)}_{t-1} - \hat{\theta}^{(s)}_{t-1} }_2 \\
    & \overset{(g)} \leq \sqrt{\beta_T} + k_2 \left(\sqrt{\lambda+(t-1)L^2}\right) q_t.
    \end{aligned}
\end{equation}
In the above steps, (a) follows from the triangle inequality; (b) follows from the definition of $\hat{\theta}^{(a)}_{t-1}$ in Eq. \eqref{eqn:non_lin_LS}; (c) follows from Lemma \ref{lemma:GLM_conf_1}; (d) follows from Lemma \ref{lemma:MVT}; (e) and (f) are both a result of the fact that $G_{t-1} \preccurlyeq k_2 V_{t-1}$; and (g) follows from the same line of reasoning as used to arrive at Eq. \eqref{eqn:conf_interim}. We have thus argued that $\theta_* \in \bar{\mathcal{C}}^{(s)}_t, \forall t\in \{\bar{T}+2, \ldots, T\}$. The rest of the proof mimics that of Lemma \ref{lemma:dec_conf_region}, and is hence omitted. 
\end{proof}

We now turn to the proof of Theorem \ref{thm:ICGLMUCB}.

\begin{proof} (\textbf{Proof of Theorem \ref{thm:ICGLMUCB}}). As in the proof of Theorem \ref{thm:ICLinUCB}, we will condition on the clean event $\mathcal{F}$, and focus on bounding the instantaneous regret at a time-step $t\geq \bar{T}+\tilde{T}+1$. To work towards this result, we define $\tilde{\theta}_t$ as
$$  (a_t, \tilde{\theta}_t) = \argmax_{(a,\theta) \in \mathcal{A}_t \times \bar{\mathcal{C}}^{(s)}_t} \mu(\langle \theta, a \rangle),
$$
where $\bar{\mathcal{C}}^{(s)}_t$ is the confidence set of the \texttt{IC-GLMUCB} algorithm as defined in Eq. \eqref{eqn:GLM_conf_region}.  Now let $a^*_t=\argmax_{a\in\mathcal{A}_t} \mu(\langle \theta_*, a \rangle)$ be an optimal action at time-step $t$. To bound the instantaneous regret $r_t = \mu(\langle \theta_*, a^*_t \rangle) - \mu(\langle \theta_*,a_t \rangle)$, we first note based on Lemma \ref{lemma:MVT} that
$$ \tilde{\theta}_t- \theta_* = G^{-1}_{t-1}\left(g_{t-1}(\tilde{\theta}_t)-g_{t-1}(\theta_*)\right). $$
This yields
\begin{equation}
\begin{aligned}
\norm[\bigg]{\tilde{\theta}_t- \theta_*}_{V_{t-1}} & = \norm[\bigg]{ G^{-1}_{t-1}\left(g_{t-1}(\tilde{\theta}_t)-g_{t-1}(\theta_*)\right)}_{V_{t-1}} \\
& = \sqrt{ \left(g_{t-1}(\tilde{\theta}_t)-g_{t-1}(\theta_*)\right)' G^{-1}_{t-1} V_{t-1} G^{-1}_{t-1}        \left(g_{t-1}(\tilde{\theta}_t)-g_{t-1}(\theta_*)\right)} \\
& \overset{(a)} \leq \sqrt{ \frac{1}{k_1} \left(g_{t-1}(\tilde{\theta}_t)-g_{t-1}(\theta_*)\right)' G^{-1}_{t-1}         \left(g_{t-1}(\tilde{\theta}_t)-g_{t-1}(\theta_*)\right)} \\
& \overset{(b)} \leq \frac{1}{k_1} \sqrt{  \left(g_{t-1}(\tilde{\theta}_t)-g_{t-1}(\theta_*)\right)' V^{-1}_{t-1}         \left(g_{t-1}(\tilde{\theta}_t)-g_{t-1}(\theta_*)\right)}  \\
& = \frac{1}{k_1} \norm[\bigg]{ g_{t-1}(\tilde{\theta}_t)-g_{t-1}(\theta_*)}_{V^{-1}_{t-1}} \\
& \overset{(c)} \leq \frac{1}{k_1} \left( \norm[\bigg]{ g_{t-1}(\tilde{\theta}_t)-g_{t-1}(\hat{\theta}^{(s)}_{t-1})}_{V^{-1}_{t-1}} + \norm[\bigg]{ g_{t-1}(\hat{\theta}^{(s)}_{t-1})-g_{t-1}(\theta_*)}_{V^{-1}_{t-1}} \right) \\
& = \frac{1}{k_1}\left(\mathcal{H}_{t-1}(\tilde{\theta}_t)+\mathcal{H}_{t-1}(\theta_*)\right)\\
& \overset{(d)} \leq \frac{2}{k_1} \left(1+4\sqrt{\frac{k_2}{k_1}} \frac{1}{\sqrt{\log(dLT)}}\right) \sqrt{\beta_T}.
\end{aligned}
\label{eqn:thm2_interim}
\end{equation}
In the above steps, (a) and (b) both follow from the fact that $G_{t-1} \succcurlyeq k_1 V_{t-1}$; (c) follows from the triangle inequality; and (d) follows by noting that $\theta_*, \tilde{\theta}_t \in \bar{\mathcal{C}}^{(s)}_t$, and by  appealing to Lemma \ref{lemma:GLM_dec_conf_region}. We now proceed to bound the instantaneous regret $r_t$ as follows.

\begin{equation}
    \begin{aligned}
    r_t &= \mu(\langle \theta_*, a^*_t \rangle) - \mu(\langle \theta_*, a_t \rangle) \\
    & \leq \mu(\langle \tilde{\theta}_t, a_t \rangle) - \mu(\langle \theta_*, a_t \rangle) \\
    & \overset{(a)} \leq k_2 \langle \tilde{\theta}_t - \theta_*, a_t \rangle\\
    & \leq k_2 {\Vert a_t \Vert}_{V^{-1}_{t-1}} \, {\Vert \theta_* - \tilde{\theta}_t \Vert}_{V_{t-1}} \\
    & \overset{(b)} \leq 2 \frac{k_2}{k_1} \left(1+4 \sqrt{\frac{k_2}{k_1}}  \sqrt{\frac{1}{\log(dLT)}}\right) {\Vert a_t \Vert}_{V^{-1}_{t-1}} \sqrt{\beta_T} \\
    & \leq 10 { \left(\frac{k_2}{k_1}\right)}^{3/2} {\Vert a_t \Vert}_{V^{-1}_{t-1}} {\sqrt{\beta_T}}. \\
    \end{aligned}
\end{equation} 
Here, (a) follows from the fact that $\mu(\cdot)$ is $k_2$-Lipschitz, and (b) follows from plugging in the  bound in Eq. \eqref{eqn:thm2_interim}. Combining the above bound with the fact that $r_t$ is trivially bounded above by $10 {(k_2/k_1)}^{3/2} \sqrt{\beta_T}$, we obtain
$$
r_t \leq 10 { \left(\frac{k_2}{k_1}\right)}^{3/2} \sqrt{\beta_T} \left(1 \wedge {\Vert a_t \Vert}_{V^{-1}_{t-1}} \right).
$$
Now following the exact same reasoning as in the proof of Theorem \ref{thm:ICLinUCB}, we can establish that on the clean event $\mathcal{F}$, 
$$ \sum_{t=\bar{T}+\tilde{T}+1}^{T}\hspace{-2.5mm} r_t \leq O\left( { \left(\frac{k_2}{k_1}\right)}^{3/2} d \sqrt{T} \log(LT) \right). $$
Moreover, since the instantaneous regret is trivially at most $1$, we have
$$ \sum_{t=1}^{\bar{T}+\tilde{T}} r_t \leq \bar{T}+\tilde{T} =  O\left( \frac{k_2}{k_1}  L^2d\sqrt{T}\log(dLT) \right) + O(\log(dLT)) = O\left( \frac{k_2}{k_1}  L^2d\sqrt{T}\log(dLT) \right).$$ 
We conclude that on the event $\mathcal{F}$ that has measure at least $1-\frac{5}{T}$, the following is true:
$$ \sum_{t=1}^{T} r_t  = 
O\left( { \left(\frac{k_2}{k_1}\right)}^{3/2} L^2 d \sqrt{T} \log(dLT) \right).
$$
The rest of the proof can be completed exactly as in Theorem \ref{thm:ICLinUCB}. 
\end{proof}
\newpage
\section{Proof of Theorem \ref{thm:ICUCB}} 
In this section, we will prove Theorem \ref{thm:ICUCB}. To get started, we introduce some notation. For each arm $i$, let us define the following observation at each $k\in[T]$:
$$ y_{i,k}= \langle \theta_*, e_i \rangle + \eta_{i,k} = \theta_i+\eta_{i,k},$$
where $\{\eta_{i,k}\}_{k\in [T]}$ is a sequence of independent 1-subgaussian random variables drawn ahead of time. In words, $\eta_{i,k}$ is the noise random variable corresponding to the $k$-th play of arm $i$. We note here that arm $i$ may not actually be played $k$ times; nonetheless, the above model offers a simple way to analyze the true dynamics. Next, we define
$$ \hat{\theta}^{(a)}_{i,k} = \frac{1}{k} \sum_{s=1}^{k} y_{i,s} $$
to be the empirical mean of $\theta_i$ (maintained by the agent) based on the first $k$ observations. We start with the following simple lemma that sets up a clean event for our subsequent analysis.

\begin{lemma}
\label{lemma:UCB_clean} Consider the following event:
$$ \mathcal{G}_i = \{ \vert \hat{\theta}^{(a)}_{i,k} - \theta_i \vert \leq f_k, \forall k \in [T] \}, \hspace{2mm} \textrm{where} \hspace{2mm} f_k = 2 \sqrt{\frac{\log T}{k}}. $$
Then, $\mathbb{P}(\mathcal{G}_i) \geq 1-\frac{2}{T}$. 
\end{lemma}
\begin{proof}
The proof is standard, and we only provide it here for completeness. Start by noting that
$$ \hat{\theta}^{(a)}_{i,k} = \frac{1}{k} \sum_{s=1}^{k} y_{i,s} = \frac{1}{k} \sum_{s=1}^{k} \left(\theta_i+\eta_{i,s} \right) = \theta_i + \frac{1}{k} \sum_{s=1}^{k} \eta_{i,s}.$$ An application of the union bound yields: 
\begin{equation}
    \begin{aligned}
    \mathbb{P}(\mathcal{G}^{c}_i) &= \mathbb{P}\left(\exists k \in [T]: \vert \hat{\theta}^{(a)}_{i,k} - \theta_i \vert > f_k \right) \\
    & \leq \sum_{k\in[T]} \mathbb{P}\left(\vert \hat{\theta}^{(a)}_{i,k} - \theta_i \vert > f_k \right)\\
    & = \sum_{k\in[T]} \mathbb{P}\left( \left| \frac{1}{k} \sum_{s=1}^{k} \eta_{i,s} \right| > f_k \right)\\
    & \leq \sum_{k\in[T]} \frac{2}{T^2} = \frac {2}{T}.
    \end{aligned}
\end{equation}
For the last step, we used the fact that $(1/k)  \sum_{s=1}^{k} \eta_{i,s}$ is a $(1/\sqrt{k})$-subgaussian random variable, and then appealed to \cite[Theorem 5.3]{tor}. 
\end{proof}

The next result tells us that with high probability, there is never any overflow during encoding.
\newpage
\begin{lemma} 
\label{lemma:UCB_enc}
Fix an action $i\in [d]$. On the event $\mathcal{G}_i$, it holds that $e_{i,n_i(t)} \in [-p_{n_i(t)}, p_{n_i(t)}], \forall t\in \{d+1, \ldots, T\}$.
\end{lemma}
\begin{proof}
With $t\in [T]$, since $n_i(t) \in [T]$, it suffices to show that on the event $\mathcal{G}_i$, it holds that
$$ e_{i,k} \in [-p_k, p_k], \forall k \in [T], $$
where $p_k$ is as given by Eq. \eqref{eqn:UCB_quant_eq}. The result follows from a simple inductive argument akin to that employed in the proof of Lemma \ref{lemma:enc_dec}. For the base case with $k=1$, we have that 
$$ e_{i,1}=\hat{\theta}^{(a)}_{i,1}-\hat{\theta}^{(s)}_{i,0} = \hat{\theta}^{(a)}_{i,1},$$ 
where we used the fact that $\hat{\theta}^{(s)}_{i,0}=0$. We thus have:
$$ |e_{i,1}| \leq \vert \hat{\theta}^{(a)}_{i,1} - \theta_i \vert + \vert \theta_i \vert \leq f_1 + m = p_1, $$
where we used the property of event $\mathcal{G}_i$, and the fact that $\max_{i\in[d]} |\theta_i| \leq m$. Now suppose $e_{i,k} \in [-p_k,p_k]$ holds for all $k\in[\ell]$, where $\ell \in [T-1]$. For the induction step, our goal is to then show that $e_{i,\ell+1} \in [-p_{\ell+1},p_{\ell+1}]$. To this end, we start by noting that 
\begin{equation}
\begin{aligned}
\left|\hat{\theta}^{(a)}_{i,\ell+1}-\hat{\theta}^{(a)}_{i,\ell}\right| & \leq \left|\hat{\theta}^{(a)}_{i,\ell+1}-\theta_i\right| + \left|\theta_i-\hat{\theta}^{(a)}_{i,\ell}\right| \\
& \leq f_{\ell+1} + f_{\ell}\\
& \leq 2 f_{\ell},
\end{aligned}
\label{eqn:UCB_induct1}
\end{equation}
where for the second inequality, we invoked the property of event $\mathcal{G}_i$. Based on the induction hypothesis, $e_{i,\ell} \in [-p_{\ell},p_{\ell}]$, i.e., there is no overflow. The encoding-decoding strategy of \texttt{IC-UCB} then yields
$$ \left|\hat{\theta}^{(a)}_{i,\ell}-\hat{\theta}^{(s)}_{i,\ell}\right| = \vert \tilde{e}_{i,\ell} - e_{i,\ell} \vert \leq \frac{1}{2^B}p_{\ell} = \gamma p_{\ell}. $$
Combining the above inequality with that in Eq. \eqref{eqn:UCB_induct1}, we obtain:
\begin{equation}
    \begin{aligned}
    |e_{i,\ell+1}| &= \left|\hat{\theta}^{(a)}_{i,\ell+1}-\hat{\theta}^{(s)}_{i,\ell}\right| \\
    & \leq \left|\hat{\theta}^{(a)}_{i,\ell+1}-\hat{\theta}^{(a)}_{i,\ell}\right| + \left|\hat{\theta}^{(a)}_{i,\ell}-\hat{\theta}^{(s)}_{i,\ell}\right|\\
    & \leq 2 f_{\ell}+ \gamma p_{\ell}\\
    & = p_{\ell+1}.
    \end{aligned}
\end{equation}
This completes the induction step and the proof. 
\end{proof}

To proceed, we will require the following intermediate result concerning the sequence $\{q_k\}$ generated as per Eq. \eqref{eqn:UCB_quant_eq}.

\begin{lemma} 
\label{lemma:sequence}
Consider the sequence $\{q_k\}$ generated as per Eq. \eqref{eqn:UCB_quant_eq}. For all $k \geq 1$, we have
\begin{equation}
    q_k \leq \gamma^{k} (m+f_1) + \frac{12}{B} \sqrt{\frac{\log T}{k}}.
\label{eqn:q_bnd_UCB}
\end{equation}
\end{lemma}
\begin{proof}
Rolling out the recursion $p_{k+1}=\gamma p_k +2 f_k$ yields:
$$ p_k = \gamma^{k-1} p_1 + 2 \sum_{s=1}^{k-1} \gamma ^{k-1-s} f_s. $$ With $p_1 =(m+f_1)$ and $q_k = \gamma p_k$, we then have
\begin{equation}
 q_k = \gamma^k (m+f_1) + 2 \sum_{s=1}^{k-1} \gamma ^{k-s} f_s = \gamma^k (m+f_1) + 4 \sqrt{\log T} \left(\sum_{s=1}^{k-1} \frac{\gamma ^{k-s}}{\sqrt{s}} \right). 
\label{eqn:sum_bnd_UCB}
 \end{equation}
 Let $a=1/\gamma$. In what follows, we will bound the following summation
$$ \sum_{s=1}^{k-1} \frac{a^s}{\sqrt{s}} \leq \underbrace{\int_{s=1}^{k} \frac{a^s}{\sqrt{s}} \, ds}_{g_k},$$
where we used the fact that $a^s/\sqrt{s}$ is monotonically increasing since $a =2^B \geq 2$. To bound the integral $g_k$, we employ a change of variable: $u=\sqrt{s}$. This yields:
$$ g_k = 2 \int_{u=1}^{\sqrt{k}} a^{u^2}\, du = 2 \int_{u=1}^{\sqrt{k}} e^{u^2 \log a} \, du,$$
where we have used $e$ to represent $\exp(1)$. Now let us employ another change of variable: $z= u\sqrt{\log a}$. We then obtain
\begin{equation}
    \begin{aligned}
    g_k &= \frac{2}{\sqrt{\log a}} \int_{\sqrt{\log a}}^{\sqrt{k\log a}} e^{z^2} \, dz \\
    & \leq \frac{2}{\sqrt{\log a}} \int_{0}^{\sqrt{k\log a}} e^{z^2} \, dz \\
    & = \frac{2}{\sqrt{\log a}} \int_{0}^{\sqrt{k\log a}} \left(\sum_{j=0}^{\infty} \frac{z^{2j}}{j!} \right) \, dz \\
    & = \frac{2}{\sqrt{\log a}} \sum_{j=0}^{\infty}  \left( \int_{0}^{\sqrt{k\log a}}  \frac{z^{2j}}{j!} \, dz \right) \\
    & = \frac{2}{\sqrt{\log a}} \sum_{j=0}^{\infty} \frac{{(\sqrt{k\log a})}^{2j+1}}{(2j+1) j!} \\
    & \leq  \frac{2}{\sqrt{k}\log a} \sum_{j=0}^{\infty} \frac{{({k\log a})}^{j+1}}{(j+1)!} \\
    & \leq  \frac{2 e^{k \log a}}{\sqrt{k}\log a}  \\
    & = \frac{2 a^k}{B \log (2) \sqrt{k}} \\
    & \leq \frac{3  a^k}{B\sqrt{k}},
    \end{aligned}
\end{equation}
where the interchange of the integral and the summation in the fourth step is warranted by the Fubini-Tonelli theorem. Plugging the above bound in Eq. \eqref{eqn:sum_bnd_UCB} and simplifying leads to the claim in Eq. \eqref{eqn:q_bnd_UCB}. 
\end{proof}

We are now ready to prove Theorem \ref{thm:ICUCB}. 

\begin{proof} (\textbf{Proof of Theorem \ref{thm:ICUCB}}) We start by defining a few quantities that will be used in our analysis. Define
$$ T_{1}= \ceil*{ \frac{1}{B} \frac{\log\left(\frac{(m+2\sqrt{\log T}) \sqrt{T}}{\sqrt{\log T}}\right)}{\log(2)}}; \hspace{2mm} T_2= \ceil*{ \frac{C}{\Delta^2_i} \log mT}, $$ 
where $C=3600.$ It is easy to verify that for $T_1 \leq t \leq T$, 
$$\gamma^t(m+f_1)= {\left(\frac{1}{2^B}\right)}^t (m+2\sqrt{\log T}) \leq \sqrt{\frac{\log T}{T}} \leq  \sqrt{\frac{\log T}{t}}.$$
Moreover, for $t \geq T_2$, we have
$$ 30 \sqrt{\frac{\log T}{t}} \leq \frac{\Delta_i}{2}. $$
Fix any arm $i$ other than arm $1$ and define $\tilde{T}=\max\{T_1, T_2\}$.\footnote{We have suppressed the dependence of $T_2$ and $\tilde{T}$ on $i$ to avoid cluttering the exposition.} Next, define the event $\mathcal{H}_{1i}= \mathcal{G}_1 \cap \mathcal{G}_i$, where recall that
$$  \mathcal{G}_i = \{ \vert \hat{\theta}^{(a)}_{i,k} - \theta_i \vert \leq f_k, \forall k \in [T] \}.$$

We claim that on the event $\mathcal{H}_{1i}$, action $i$ will be played at most $\tilde{T}$ times, i.e., $n_i(T) \leq \tilde{T}$. To establish this claim, we proceed via contradiction. Accordingly, suppose that on the event $\mathcal{H}_{1i}$, $n_i(T) > \tilde{T}$. Thus, there must exist a time-step $t\in [T]$ such that $n_i(t-1) = \tilde{T}$, and $a_t=e_i$. At this time-step, we have
\begin{equation}
    \begin{aligned}
    \texttt{IC-UCB}_i(t-1)& = \hat{\theta}^{(s)}_{i,n_i(t-1)}+q_{n_i(t-1)}+f_{n_i(t-1)} \\
    &= \hat{\theta}^{(s)}_{i,\tilde{T}}+q_{\tilde{T}}+f_{\tilde{T}} \\
    & \overset{(a)} \leq \hat{\theta}^{(a)}_{i,\tilde{T}} + 2 q_{\tilde{T}} + f_{\tilde{T}} \\
    & \overset{(b)} \leq \theta_i + 2 \left( q_{\tilde{T}} + f_{\tilde{T}} \right) \\
    & \overset{(c)} \leq \theta_i + 30 \sqrt{\frac{\log T}{\tilde{T}}} \\
    & \overset{(d)} \leq \theta_i + \frac{\Delta_i}{2}\\
    & < \theta_1 \\
    & \overset{(e)} \leq \texttt{IC-UCB}_1(t-1).
    \end{aligned}
\end{equation}
We have thus arrived at a contradiction as $a_t=e_i \neq \argmax_{j\in [d]} \texttt{IC-UCB}_j (t-1)$. To complete the proof of the claim, we need to justify each of the above steps. For (a), we invoked Lemma \ref{lemma:UCB_enc} to conclude that
$$ \left|\hat{\theta}^{(a)}_{i,\tilde{T}}-\hat{\theta}^{(s)}_{i,\tilde{T}}\right| = \vert \tilde{e}_{i,\tilde{T}} - e_{i,\tilde{T}} \vert \leq  \gamma p_{\tilde{T}} = q_{\tilde{T}}.$$ For (b), we used the defining property of event $\mathcal{G}_i$. For (c), we appealed to Lemma \ref{lemma:sequence}, and used the facts that $\tilde{T} \geq T_1$ and $B \geq 1$ to conclude that 
$$ q_{\tilde{T}} \leq \gamma^{\tilde{T}} (m+f_1) + \frac{12}{B} \sqrt{\frac{\log T}{\tilde{T}}} \leq 
13 \sqrt{\frac{\log T}{\tilde{T}}}.
$$ 
It remains to argue that on the event $\mathcal{H}_{1i}$ (that contains the event $\mathcal{G}_1$), 
$$ \theta_1 \leq \texttt{IC-UCB}_1(t-1). $$
We claim that the above inequality holds for all $t\in [T]$. Suppose by contradiction that there exist some $t, k \in [T]$ such that $n_i(t-1)=k$, and 
$$ \theta_1 > \hat{\theta}^{(s)}_{i,k}+q_{k}+f_{k}. $$ Based on the property of event $\mathcal{G}_1$, this would imply that
$$ \hat{\theta}^{(a)}_{i,k}+f_{k} > \hat{\theta}^{(s)}_{i,k}+q_{k}+f_{k} \implies \hat{\theta}^{(a)}_{i,k} - \hat{\theta}^{(s)}_{i,k} > q_k, $$
which is a contradiction since based on Lemma \ref{lemma:UCB_enc},
$$ \left|\hat{\theta}^{(a)}_{i,k}-\hat{\theta}^{(s)}_{i,k}\right|  \leq  q_{k}. $$

We have thus established the claim that on the event $\mathcal{H}_{1i}$, $n_i(T) \leq \tilde{T}$. From Lemma \ref{lemma:UCB_clean}, we also note that 
$$ \mathbb{P}(\mathcal{H}_{1i}) \geq 1-\frac{4}{T}.$$ This immediately leads to the following bound:
\begin{equation}
\begin{aligned}
    \mathbb{E}[n_i(T)] &= \mathbb{E}[\mathcal{I}_{\mathcal{H}_{1i}} n_i(T)]+ \mathbb{E}[\mathcal{I}_{\mathcal{H}^c_{1i}} n_i(T)]\\ 
    & \leq \tilde{T}+ T \mathbb{P}(\mathcal{H}^{c}_{1i})\\
    & \leq \tilde{T}+4.
\end{aligned}
\end{equation}
Using the facts that $m \geq 1$, $B \geq 1$, and $\Delta_i \leq 1$, one can verify that $T_1 \leq T_2$. Hence, 
$$ \mathbb{E}[n_i(T)] \leq \frac{C}{\Delta^2_i} \log (mT) +5.$$ 
This immediately yields the desired regret bound:
$$ R_T = \sum_{i=1}^{d} \Delta_i \mathbb{E}[n_i(T)] \leq 5 \sum_{i=1}^{d} \Delta_i + \sum_{i=1}^{d} \frac{C}{\Delta_i} \log (mT). $$ 
\end{proof}

We now comment on the proof of Theorem \ref{thm:gap_ind}.

\begin{proof} (\textbf{Proof of Theorem \ref{thm:gap_ind}}) Starting from the bound 
$$ \mathbb{E}[n_i(T)] \leq \frac{C}{\Delta^2_i} \log (mT) +5$$
that we derived in the analysis of Theorem \ref{thm:ICUCB}, the rest of the proof of Theorem \ref{thm:gap_ind} follows exactly the same reasoning as \cite[Theorem 7.2]{tor}. Hence, we omit the details. 
\end{proof}
\end{document}